\documentclass[10pt]{article}

\usepackage{amsmath,amsthm,verbatim,amssymb,amsfonts,amscd, graphicx, enumitem}
\usepackage{graphics}
\usepackage{centernot}
\usepackage{authblk}
\theoremstyle{plain}
\newtheorem{theorem}{Theorem}

\newtheorem{conjecture}{Conjecture}
 
\newtheorem{definition}{Definition}

\usepackage[colorlinks,linkcolor=blue,citecolor=blue]{hyperref}

\usepackage[bottom]{footmisc}
\usepackage{caption}
\usepackage{subcaption}
\usepackage{helvet}  
\usepackage{courier}  
\usepackage{url}  
\usepackage{graphicx}  
\usepackage{multirow}
\usepackage{amsthm}
\usepackage{color}
\usepackage{MnSymbol}
\usepackage{makecell}
\usepackage{arydshln}
\usepackage{amsmath}
\usepackage[dvipsnames]{xcolor}
\usepackage{caption} 
\usepackage{bbm}

\usepackage{textcomp}
\usepackage{wrapfig}
\usepackage{algorithm}
\usepackage{algorithmic}

\usepackage{csquotes}

\usepackage[many]{tcolorbox}    	
\usepackage{setspace}               
\usepackage{multicol}               

\definecolor{main}{HTML}{5989cf}    
\definecolor{sub}{HTML}{cde4ff}     

\tcbset{
    sharp corners,
    colback = white,
    before skip = 0.2cm,    
    after skip = 0.5cm      
}                           

\newtcolorbox{boxA}{
    fontupper = \bf,
    boxrule = 1.5pt,
    colframe = black 
}

\newtcolorbox{boxB}{
    fontupper = \bf\color{main}, 
    boxrule = 1.5pt,
    colframe = main,
    rounded corners,
    arc = 5pt   
}

\newtcolorbox{boxC}{
    colback = sub, 
    boxrule = 0pt  
}

\newtcolorbox{boxD}{
    colback = sub, 
    colframe = main, 
    boxrule = 0pt, 
    toprule = 3pt, 
    bottomrule = 3pt 
}

\newtcolorbox{boxE}{
    enhanced, 
    boxrule = 0pt, 
    borderline = {0.75pt}{0pt}{main}, 
    borderline = {0.75pt}{2pt}{sub} 
}

\newtcolorbox{boxF}{
    colback = sub,
    enhanced,
    boxrule = 1.5pt, 
    colframe = white, 
    borderline = {1.5pt}{0pt}{main, dashed} 
}

\newtcolorbox{boxG}{
    enhanced,
    boxrule = 0pt,
    colback = sub,
    borderline west = {1pt}{0pt}{main}, 
    borderline west = {0.75pt}{2pt}{main}, 
    borderline east = {1pt}{0pt}{main}, 
    borderline east = {0.75pt}{2pt}{main}
}

\newcommand{\E}{\mathbb{E}}

\begin{document}

\title{Parameter Symmetry Potentially Unifies Deep Learning Theory
}

\author{Liu Ziyin$^{1,3}$, Yizhou Xu$^2$, Tomaso Poggio$^1$, Isaac Chuang$^1$\\
$^1$\textit{Massachusetts Institute of Technology}\\
$^2$\textit{École Polytechnique Fédérale de Lausanne}\\
$^3$\textit{NTT Research}
}
\maketitle

\begin{abstract}
    The dynamics of learning in modern large AI systems is hierarchical, often characterized by abrupt, qualitative shifts akin to phase transitions observed in physical systems. While these phenomena hold promise for uncovering the mechanisms behind neural networks and language models, existing theories remain fragmented, addressing specific cases. In this position paper, we advocate for the crucial role of the research direction of parameter symmetries in unifying these fragmented theories. This position is founded on a centralizing hypothesis for this direction: parameter symmetry breaking and restoration are the unifying mechanisms underlying the hierarchical learning behavior of AI models. We synthesize prior observations and theories to argue that this direction of research could lead to a unified understanding of three distinct hierarchies in neural networks: learning dynamics, model complexity, and representation formation. By connecting these hierarchies, our position paper elevates symmetry -- a cornerstone of theoretical physics  -- to become a potential fundamental principle in modern AI. 
\end{abstract}

\section{Introduction}
More and more universal emergent phenomena of learning have been discovered in contemporary AI systems, and there is an urgent need for a unified theory to understand and explain them. These phenomena are shared by models with different architectures, trained on different datasets, and with different training techniques. The existence of these universal phenomena calls for one or a few universal explanations. However, until today, most of the phenomena are instead described by narrow theories tailored to explain each phenomenon separately -- often focusing on specific models trained on specific tasks or loss functions and in isolation from other interesting phenomena that are indispensable parts of the deep learning phenomenology. Certainly, it is desirable to have a \textit{universal perspective}, if not a \textit{universal theory}, that explains as many phenomena as possible. In the spirit of science, a universal perspective should be independent of system details such as variations in minor architecture definitions, choice of loss functions, training techniques, etc. A universal theory would give the field a simplified paradigm for thinking about and understanding AI systems and a potential design principle for a new generation of more efficient and capable models.

Learning phenomena in deep learning can be roughly categorized into three types, each capturing a different kind of hierarchy hidden in neural networks:
\begin{itemize}[noitemsep,topsep=0pt, parsep=0pt,partopsep=0pt, leftmargin=12pt]
    \item \textbf{Hierarchy of Learning Dynamics}: distinct \textit{temporal} regimes arise during learning -- this includes abrupt complexity jumps \cite{simon2023stepwise, jacot2021saddle, abbe2023sgd}, progressive sharpening and flattening \cite{cohen2021gradient}, and beyond-linear dynamics of training \cite{zhu2022quadratic};
    \item \textbf{Hierarchy of Model Complexity}: the \textit{functional} complexity of models adapts to the target function -- this is exhibited in simplicity biases \cite{kalimeris2019sgd}, compressive coding through the information bottleneck \cite{tishby2000information, tishby2015deep}, and the ``blessing of dimensionality" in overparameterized nets \cite{galanti2021role, Zhang_rethink};
    \item \textbf{Hierarchy of Neural Representation}: distinct \textit{spatial} structures arise in the layers of neural networks, with progressively deeper layers tending to encode increasingly abstract information -- this is evidenced in the structured representations such as neural collapse \cite{papyan2020prevalence}, hierarchical encoding of features \cite{zeiler2014visualizing}, and universal alignment of representations across models \cite{huh2024platonic}.
\end{itemize}

\begin{figure*}
    \centering
    \includegraphics[width=0.99\linewidth]{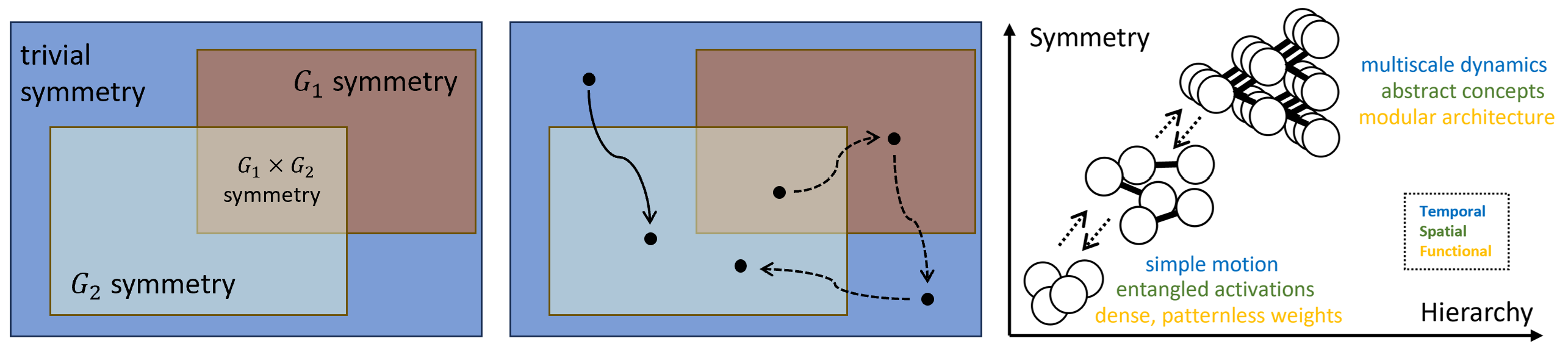}
    \caption{\small The division of solution space into hierarchies given by distinct parameter symmetries. \textbf{Left}: Example solution space of a model with parameter symmetries can be divided into hierarchies with boundaries prescribed by symmetry-breaking conditions. The more symmetry there is, the more restricted the hypothesis space becomes. \textbf{Middle}: The temporal (learning dynamics) and spatial (layer-wise information processing) dynamics of AI models can be characterized by the transitions between different symmetry groups. The solid line shows a symmetry restoration dynamics when the parameter transitions from a low-symmetry state to a high one (through time or through layers). The dashed lines show a compositional dynamics where the model follows two symmetry breaking and then a restoration. Actual learning dynamics of neural networks may involve the model first learning by breaking symmetries before regularization effects dominate \cite{li2021happens}, restoring symmetry. Similarly, for spatial processing, neural networks are found to break symmetries in early layers and restore symmetries in final layers (Section~\ref{sec: representation}). \textbf{Right}: The more symmetry a model has, the more spatial, temporal, and functional hierarchies it has. Changes in symmetries can induce transitions between these hierarchies.}
    \label{fig:main illustration}
\end{figure*}

In this paper, we argue for the position that the parameter symmetries should be a crucial perspective and direction of research in order to build universal theories of AI because, as we will collect key prior results to demonstrate, these three seemingly unrelated hierarchies of deep learning phenomena can be collectively understood as arising from having parameter symmetries in the model or loss function. Our central position is:
\begin{boxC}
\paragraph{Position:} Unifying theories and phenomenology of AI systems requires parameter symmetries, \textit{because} the hierarchies of learning dynamics, model complexity, and representation formation are primarily determined by parameter symmetry breaking and restoration.
\end{boxC}
Prior art has certainly hinted at this idea, but here we draw together many separate threads to complete a single picture, as illustrated in Figure~\ref{fig:main illustration}. The main position is supported by three specific hypotheses, each of which is now extensively supported by prior results:
\begin{itemize}[noitemsep,topsep=0pt, parsep=0pt,partopsep=0pt, leftmargin=14pt]
    \item \textit{Temporal} jumps follow transitions between symmetry groups (Section~\ref{sec: dynamics});
    \item \textit{Functional} constraints arise adaptively along boundaries of symmetry (Section~\ref{sec: complexity});
    \item \textit{Spatial} hierarchy of abstract representations arise due to layered symmetries (Section~\ref{sec: representation}).
\end{itemize}
If parameter symmetry is really  a unifying mechanism for deep learning, our understanding of the principles of AI can advance with explicit identification and engineering of symmetries in models, loss functions, and data. With the recent advances in how arbitrary parameter symmetries may be deliberately introduced or removed (Section~\ref{sec: mechanisms}), it is now possible to design symmetries matching practitioner intentions. Moreover, given widespread beliefs that hidden hierarchy phenomena of the kinds described above are beneficial, it follows one should be able to design models with intentional symmetries to introduce hierarchies into learning dynamics, the learned function complexities, and the layered structure of neural networks.

The next section introduces parameter symmetry. We then discuss existing supports for the three specific hypotheses given above. Section~\ref{sec: mechanisms} discusses mechanisms that cause and methods for controlling symmetry breaking and restoration. The last section discusses alternative views. Proofs, experimental details, andadditional figures are presented in the appendix.

\begin{table*}[t!]
    \centering
    \small
    \begin{tabular}{c|ccc}
    \hline
        Symmetry & Model condition  & Symmetric State  & Example \\
        \hline
        translation & $f(w) = f( w + \lambda z)$ for fixed $z$ & none &  softmax, low-rank inputs \\ 
       scaling & $f(w) = f(\lambda w)$ & none & batchnorm, etc. \\ 
       rescaling   & $f(u,w) = f(\lambda u, \lambda^{-1} w)$  & $\|u\|=\|w\|$ &  ReLU neuron \\
       rotation   & $f(W) = f(RW)$ for orthogonal $R$  & low-rank solutions & self-supervised learning \\
       permutation  &  $f(u,w) = f(w,u)$ & identical neurons & dense layers, ensembles\\
       double rotation  &  $f(U,W) = f(UA,A^{-1}W) $ & low-rank solutions  
 & self-attention, linear nets\\
       sign flip & $f(w) = f(-w) $ & $w=0$ & tanh neuron\\
    \hline
    \end{tabular}
    \caption{\small Common parameter symmetries in deep learning. We divide $\theta$ into three parts: $\theta=(w,u,v)$, where $w$ and $u$ are related to symmetry, while $v$ is symmetry-irrelevant and is omitted. Note that these symmetries are not mutually exclusive. For example, double rotation or rotation symmetry implies permutation symmetry and sign flip. Double rotation also implies rescaling. Some continuous groups do not have a discrete subgroup, such as the scaling and translation symmetry, which is also included for completeness. However, they still interact with regularizations in an interesting way: If there is a weight decay, the global minima are achieved at zero, which is ill-behaved for scaling symmetry but not for translation. Also, note that $Z_2$ subgroups are particularly common in these symmetries.}
    \label{tab:list of symmetries}
\end{table*}

\section{Parameter Symmetry in Deep Learning}
\label{sec:psymdef}

\begin{definition}
    Let $G$ be a linear representation of a group. We say that there is a $G$-parameter symmetry in the model $f(\theta,x)$ if $\forall g\in G$ and $\forall x$, $f(\theta, x) = f(g\theta, x)$.
\end{definition}
\begin{definition}\label{def: symmetry breaking}
    Let $P_G= \frac{1}{|G|}\sum_{g\in G}g $. The model parameter $\theta$ is said to be in a $G$-symmetric state if $P_G \theta = \theta$. Otherwise, $\theta$ is in the symmetry-broken state.
\end{definition}

These definitions are adopted from Ref.~\cite{ziyin2024remove}. $P_G$ is the projection matrix to the symmetry-invariant subspace of $G$ \cite{feit1982representation}. Note that if $P_G \theta = \theta$, then $g\theta = \theta$ for any $g$ because $gP_G=P_G$. A striking fact for those unfamiliar with parameter symmetries is that many such parameter symmetries exist in widely used loss functions and neural networks (see Table~\ref{tab:list of symmetries}). Understanding the implications of such symmetries has been an active area of research \cite{simsek2021geometry, entezari2021role, Dinh_SharpMinima, saxe2013exact, neyshabur2014search, tibshirani2021equivalences, ioffe2015batch,zhao2023improving,godfrey2022symmetries}, and this has been challenging since the number of groups induced by these symmetries often grows exponentially in the size of the model. Many prior works have explored the effect of parameter symmetry on constraining the model complexity or learning dynamics \cite{li2016symmetry, saxe2013exact, du2018algorithmic, hidenori2021noether, zhao2022symmetry, marcotte2023abide,li2020reconciling}. Because parameter symmetry naturally means that part of the solution space is redundant, having parameter symmetry can be seen as a form of ``overparameterization."

\paragraph{Example:} Consider the parameter symmetries in the self-attention logit of a transformer, which can be written as (using notation from Ref.  \cite{vaswani2017attention})
\begin{equation}
    a_{ij}(W_Q, W_K) = X_i^T W_Q W_K X_j. 
\end{equation}
It has a double rotation symmetry: for any invertible $M$, $a_{ij}(W_Q M, M^{-1}W_K) =a_{ij}(W_Q, W_K)$, which implies that the loss function is also invariant to this transformation. This, in turn, implies that the symmetric states are those with low-rank $W_Q$ and $W_K$. Note that $a_{ij}$ can also be written as 
\begin{equation}
    a_{ij}(W_Q, W_K) = X_i^T \sum_{l} W_Q^l (W_K^l)^T X_j,
\end{equation}
where $W_Q^l$ is the $l$-th column of $W_Q$ and $W_K^l$ is the $l$-th row of $W_K$. This means that it also has the permutation symmetry such that for any $l$ and $l'$, $a_{ij}(W_Q^l, W_K^l, W_Q^{l'}, W_K^{l'}) = a_{ij}(W_Q^{l'}, W_K^{l'}, W_Q^{l}, W_K^{l})$. An alternative way to see this is that permutation symmetry is a subgroup of the double rotation symmetry (because permutation matrices are invertible). For each $l$, there is also the sign flip symmetry, $a_{ij}(W_Q^l, W_K^l) = a_{ij}(-W_Q^l, -W_K^l)$, and the rescaling symmetry, $a_{ij}(\lambda W_Q^l, \lambda^{-1}W_K^l)$. Note that if $X$ is low-rank in the $z$-direction, these matrices also have a translation symmetry, where $a_{ij}(W_K^l + \lambda z)= a_{ij}(W_K^l)$ for any $l$ and $\lambda \in \mathbb{R}$. Therefore, we have seen that even for a single self-attention layer, there are many symmetries. 
Note that these symmetries are independent of the data distribution and apply to every pair of $X_i$ and $X_j$ -- this is the main reason these symmetries could significantly affect training neural networks. Recent works have shown that double rotation symmetry causes the self-attention layers to have a low-rank bias \cite{kobayashi2024weight}.

\section{Learning Dynamics is Symmetry-to-Symmetry}\label{sec: dynamics}

Symmetry has a direct influence on the loss landscape, and it thus affects the learning dynamics of neural networks through its effect on the landscape \cite{tatro2020optimizing, lim2024empirical}. One primary effect of symmetry on the loss landscape is that it creates extended saddle points from which SGD or GD cannot escape \cite{li2016symmetry, chen2023stochastic}. Meanwhile, it has been found that when a neural network is initialized with small norm weights, its learning dynamics is primarily saddle-to-saddle \cite{jacot2021saddle}. In fact, neural networks have been found to converge often to saddle points~\cite{alain2019negative, ziyin2023probabilistic}. Given that symmetries are the primary origins of saddle points, it is natural to hypothesize that the learning dynamics of neural networks are not only saddle-to-saddle but symmetry-to-symmetry: 
\begin{boxC}
\paragraph{Dynamics Hypothesis:} The learning dynamics of neural networks are dominated by jumps between symmetry groups, with parameters going from a larger to a smaller group (symmetry breaking) or from a smaller to a larger group (restoration).
\end{boxC}

As Theorem~\ref{theo: complexity} in the next section shows, at a $G$-symmetric solution, the number of effective model parameters is reduced by a number that matches the rank of the group. This means the change between symmetries naturally induces a change in model complexities. Therefore, these symmetry-to-symmetry jumps not only in terms of the loss function value but also in terms of complexity jumps.

Definition~\ref{def: symmetry breaking} can quantify the breaking and restoration of symmetry. 
We define the symmetry-breaking distance as:
\begin{equation}
    \Delta^G = \left \|\theta - P_G \theta \right\|_2^2 .
\end{equation}
When $\Delta^G > \Delta^G_{\rm th}$ for some threshold $\Delta^G_{\rm th}$ ($\Delta^G_{\rm th} = 0.05\sim 0.2$ in experiments), we say the $G$-symmetry is broken. We care about how many symmetries are broken for a given layer, so we can count the number of such large $\Delta^G$. For example, for permutation symmetry in a fully connected layer, the group-invariant projection is the average of the input and output weights of any pair of neurons $i$ and $j$. One can identify each pairwise distance as a different $\Delta^G$, which we will denote as $\Delta_{ij}$ throughout this work. We only need to count the number of neighboring neurons with a large $\Delta^G$ because they form a generating set of the symmetric group -- we refer to this number as the \textit{degree of symmetry} $N_{\rm dos}$. The difference between $N_{\rm dos}$ and the number of neurons is the \textit{degree of symmetry breaking} ($N_{\rm dosb}$). See Section~\ref{app sec: measurement} for details on measuring this quantity.

\begin{figure}
    \centering
    \includegraphics[width=0.9\linewidth]{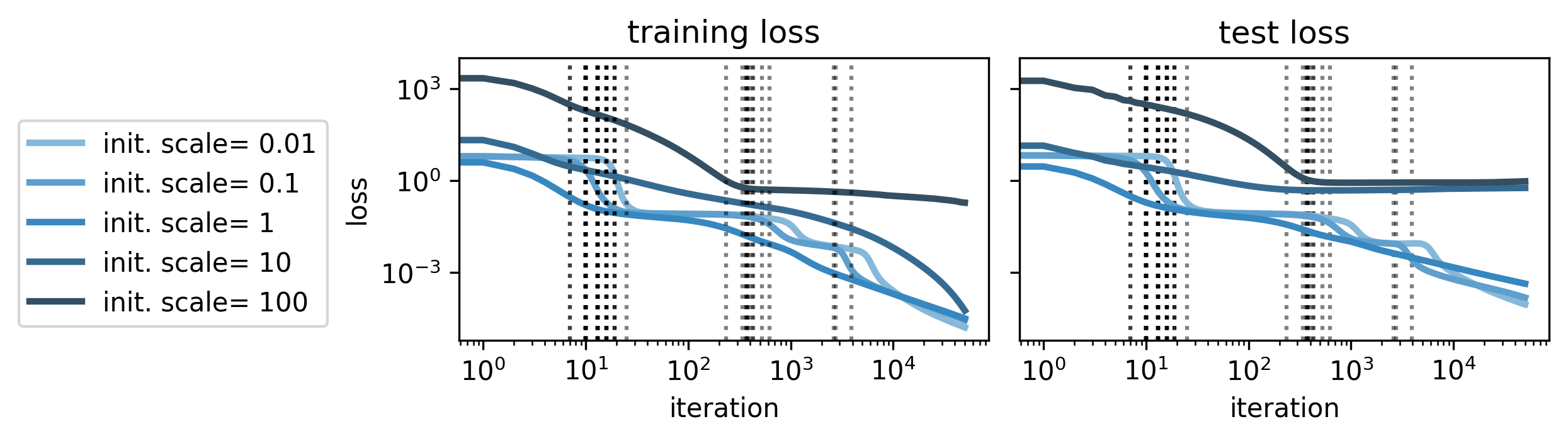}
    \caption{\small DNN learning dynamics is symmetry-to-symmetry. Recent works suggested the learning of neural networks is primarily saddle-to-saddle \cite{jacot2021saddle}, and escaping these saddle points coincides with a sudden change in the complexity of the network \cite{abbe2023sgd}. At the same time, symmetries have been found to be the primary causes of the saddle points \cite{li2016symmetry, ziyin2024symmetry}. Once the symmetry is removed, saddle points seem to have disappeared when interpolating different solutions of the model \cite{lim2024empirical}. The figure repeats an experiment that is similar to those in Ref. \cite{abbe2023sgd} and shows that the loss jumps when symmetry breaking happens (black dotted lines) and plateaus when there is no symmetry breaking for the smallest init. As the init. scale becomes larger; such plateaus disappear because they are far away from symmetric states.}
    \label{fig:symmetry breaking dynamics}
\end{figure}

Figure~\ref{fig:symmetry breaking dynamics} shows an example where we train an MLP in a teacher-student setting using SGD. The model exhibits multiple symmetry jumps when we initialize the model at a small norm, i.e. the model's initial state is approximately in the symmetric state. The black dashed lines show when a pairwise distance exceeds $\Delta^{G}_{\rm th}$, signaling a symmetry breaking. We see that the symmetry-breaking times coincide well with the periods where the network learns rapidly. This shows that the saddle-to-saddle learning dynamics is potentially caused by the symmetry-to-symmetry transitions during training. Examples with polynomial and sinusoidal nets are shown in Section~\ref{app sec: exp detail}.


\paragraph{Decomposability and Feature Learning} Another interesting point worth raising about symmetries is that if the model is close to a symmetric state $\theta^*$, then it can be decomposed into a sum of a quadratic model and a strictly smaller model. Ref. \cite{ziyin2024symmetry} showed that for a small $\Delta^G$,
\begin{equation}
    f(\theta,x) = f(\theta^*) + \theta_G^T \nabla^2 
 f(\theta^*,x)  \theta_G + o(\Delta^G),
\end{equation}
where $\theta^* = P_G \theta$ is essentially a parameter projected to a $d-{\rm rank}(P_G)$ subspace, and $\theta_G = (I-P_G)\theta$ are the parameters projected to the symmetry breaking direction. Therefore, the second term can be seen as a quadratic model trained on a data transformation kernel of the form $ \nabla^2 f(x)$. This means that close to symmetry states, the learning dynamics must be ``diligent" and cannot be characterized by NTK \cite{jacot2018neural} or lazy training \cite{chizat2018lazy}. Thus, parameter symmetry also plays a significant role in feature learning \cite{yang2020feature} and a lot of work on neural network training with small initialization reflect the roles of symmetry in feature learning \cite{saxe2013exact, kunin2024get, gissin2019implicit, atanasov2021neural, simon2023stepwise, xu2024does}. In contrast, if a model is antisymmetric, its expansion will only have odd-order terms, and one expects that this model primarily behaves like a linear model.

\begin{figure*}[t!]
    \centering
    \hfill
    \includegraphics[width=0.26\linewidth]{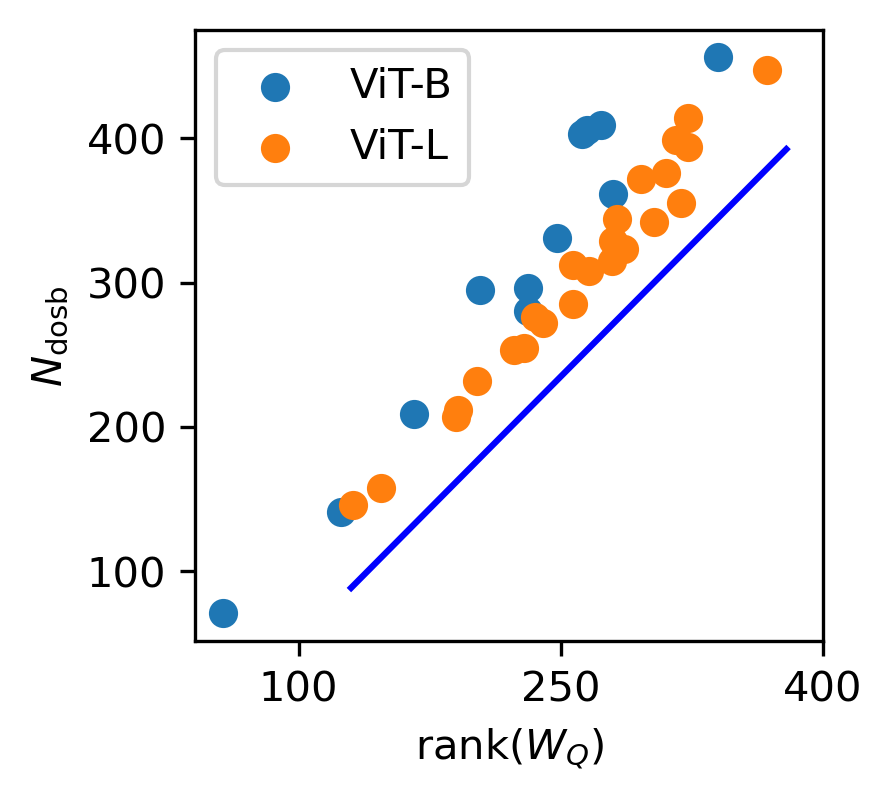}
    \
    \includegraphics[width=0.35\linewidth]{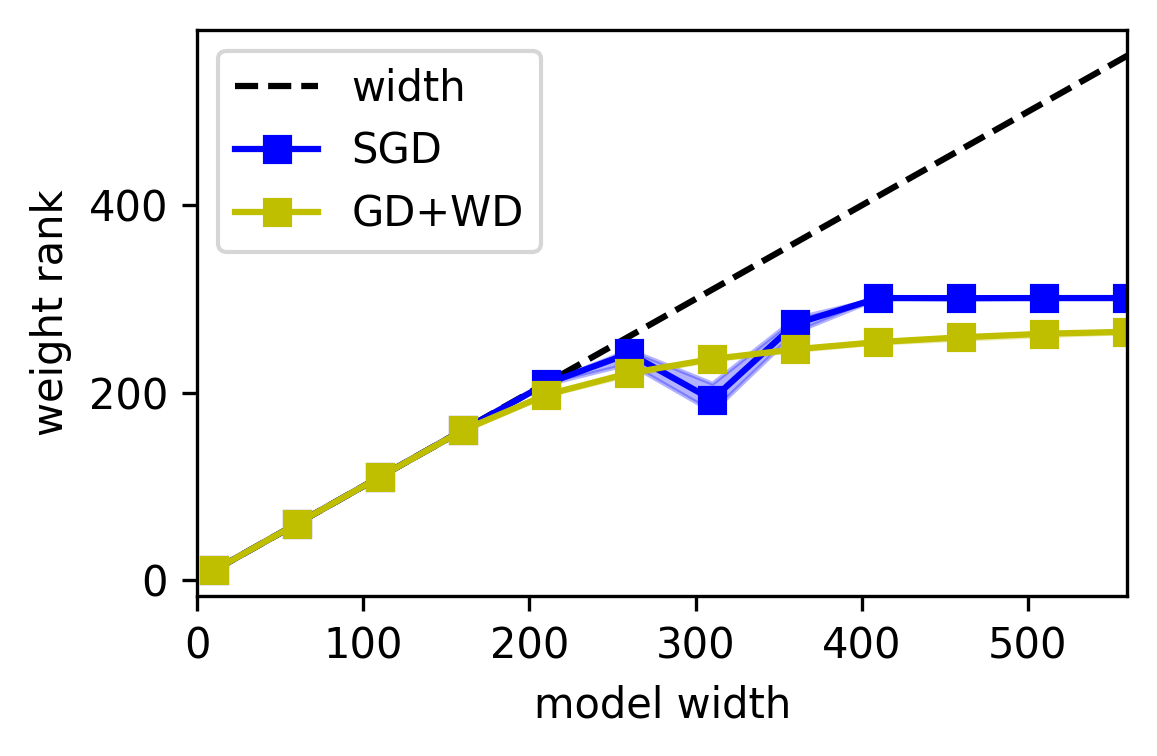}
    \hfill
    \includegraphics[width=0.23\linewidth]{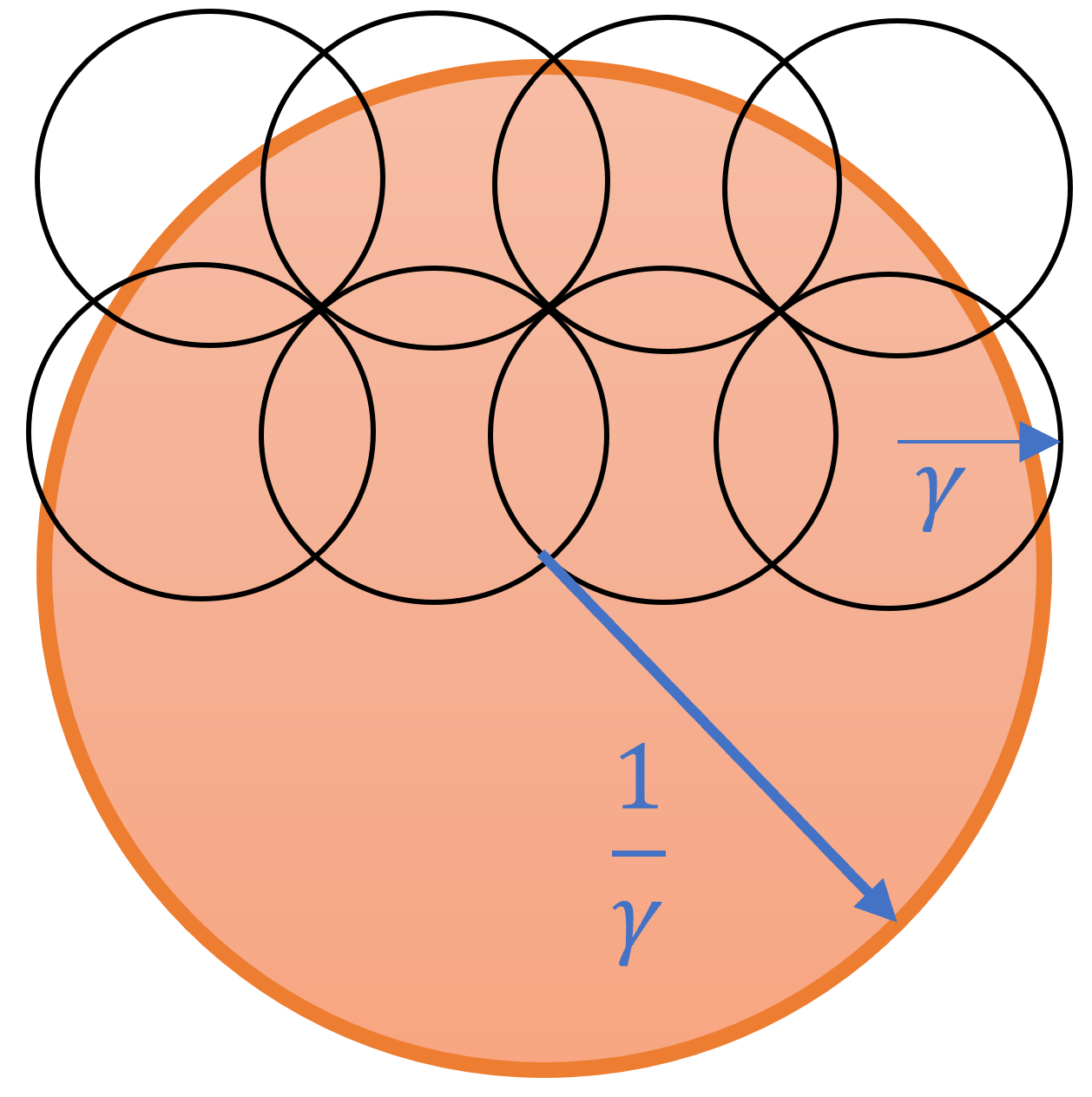}
    \hfill
    \caption{\small The complexity and generalization error of neural networks do not grow with width. A well-known observation in deep learning is that overparameterized networks not only work well, but also their generalization errors are empirically found to be essentially independent of width \cite{li2018neural, pinto2024generalization, galanti2023norm, mingard2025deep}, an observation at odds with conventional bounds based on the Rademacher complexity \cite{Zhang_rethink}. The existence of parameter symmetries may solve this problem because, with a fixed regularization, the maximum surviving neurons are upper bounded by a constant. The \textbf{left} figure takes publically available Imagenet-pretrained ViT-Base (80M) and ViT-Large (300M) have similar degrees of symmetry in their self-attention layers. Here, each dot is a self-attention layer. The \textbf{middle} figure shows that the weight rank does not grow with increasing model size. A mechanism for this is that permutation symmetry implies that neuron weights of distance $o(\gamma)$ to each other must collapse \cite{xu2025unpub1} -- this means that within a fixed $n$-sphere, there can be at most $1/\gamma^n$ different neurons. This filling procedure is illustrated in the \textbf{right} figure: The orange circle denotes the parameter space, and the little circles are the neuron weights.}
    \label{fig:complexity upper bound}
\end{figure*}
\section{Symmetry Adaptively Limits Model Complexity}\label{sec: complexity}

Another important implication of symmetries is that they control the effective number of parameters. When a system is symmetry-broken, new degrees of freedom usually emerge (known as a Goldstone mode in physics \cite{peskin2018introduction}), and the system becomes higher-dimensional and more complex. When it is symmetry-restored, the effective dimension is reduced. This means that the symmetry boundaries naturally correspond to boundaries of different hierarchies of model capacities. This idea has a place in machine learning. For example, equivariant networks can improve the sample efficiency of training, and existing equivariant networks almost always involve introducing new parameter symmetry to the neural network \cite{maron2018invariant, bronstein2021geometric}. In Bayesian learning, parameter symmetry (and its generalizations) have been found to directly determine the generalization scaling of the model \cite{watanabe2010asymptotic}. Prior works also indicate how neural networks may break symmetries to adapt to different tasks \cite{fok2017spontaneous}.

More broadly, combining the idea that symmetry classes have different model complexity and the common observation that SGD tends to learn a function whose complexity is proportional to the complexity of the target function \cite{kalimeris2019sgd, mingard2025deep}, it is natural to arrive at the following hypothesis:
\begin{boxC}
\paragraph{Complexity Hypothesis:}  Symmetry adaptively controls the model's capacity. The model converges to a symmetry class whose complexity matches the complexity of the target.
\end{boxC}
The direct correspondence between symmetry and model capacity has been justified by a recent result: at a $G$-symmetric state, the effective model dimension decreases by exactly ${\rm rank}(P_G)$ throughout training. Moreover, in the NTK limit, this reduction in parameter dimension implies a selection of input features, and being at a symmetric solution directly affects the input-output function map.
\begin{theorem}\label{theo: complexity}
    (Informal, \cite{ziyin2024remove}) If the loss function has $G$-symmetry, and the initial $\theta\in \mathbb{R}^d$ is $G$-symmetric, then (1) there exists a model with $d-{\rm rank}(P_G)$ parameters whose learning dynamics is the same as $\theta$, and (2), in the lazy training regime, this is equivalent to applying a rank $d-{\rm rank}(P_G)$ mask to the NTK features.
\end{theorem}
Thus, symmetric solutions are low-capacity states from which gradient-based training methods cannot escape. More importantly, these symmetric solutions are preferred solutions when weight decay is used \cite{ziyin2024symmetry} or if the minibatch noise is strong due to a mechanism called ``stochastic collapse" \cite{chen2023stochastic}. 

\paragraph{Parameter Symmetry as an Occam's Razor} In other words, parameter symmetries in the model may function like an Occam's razor, where simpler and low-complexity solutions are favored -- an observation that has been made across almost all modern neural models, whose generalization performances are independent of the size of the model \cite{Zhang_rethink, galanti2023norm}, which is in discrepancy with common generalization bounds that predict a $\sqrt{\psi}$ deterioration with respect to the width $\psi$ \cite{neyshabur2018towards}. One can test this hypothesis simply by training an MLP and a small transformer on different randomly generated teacher networks -- see Section~\ref{app sec: adaptive capacity}. For different tasks, the same network converges to solutions with different symmetry classes corresponding to different levels of complexity. Similarly, for a fixed task, networks with different levels of capacity converge to similar degrees of symmetry. See Figure~\ref{fig:complexity upper bound}-Left for vision transformers (ViT) \cite{dosovitskiy2020image} trained on Imagenet.

Here, one can raise an insightful conjecture about the quantification of complexity control due to permutation symmetry, where $\gamma$ is the weight decay,  $\eta/S$ is the learning-rate-batch-size ratio. Recall that the distance between two neurons is $\Delta_{ij}$.
\begin{conjecture}
(Space Quantization Conjecture \cite{xu2025unpub1}, Informal) In every layer with permutation symmetry, for two nonidentical neurons $i$ and $j$, $\Delta_{ij} > O(\kappa^{\beta})$ after training, where $\beta >0$ and $\kappa$ is the regularization strength.
\label{conjecture}
\end{conjecture}
For completeness, a special case is proved in Section~\ref{app sec: space quantization conjecture}. A primary mechanism for this absolute upper limit is a combination of regularization and symmetry (Section~\ref{sec: mechanisms}): (1) using weight decay $\gamma$, the model parameters must lie within a ball of radius $O(1/\gamma)$ to the origin; (2) with symmetry any neuron whose weights are with a distance of $\gamma$ to each other must become identical. This means that the number of surviving neurons is upper bounded by the number of $n-$hyperspheres of radius $O(\gamma)$ that one can fit in a hypersphere of radius $O(1/\gamma)$; this number can be upper bounded by $O(\gamma^{-n})$. This exponential index can be improved to be independent of $n$ and so does not suffer the curse of dimensionality (Section~\ref{app sec: space quantization conjecture}). 
From a physics picture, the neurons start to form a ``lattice" in the parameter space with the lattice distance $\gamma$.

Thus, when regularized, there are, at most, a finite number of nonidentical neurons within a layer, however wide it is. This implies that the actual complexity of the trained network must decrease as a function polynomial in the regularization strength. 
See Figure~\ref{fig:complexity upper bound}-Middle, where we train a two-hidden-layer network with varying widths. When weight decay or stochastic training is used, we see that the dimension of the latent representation saturates at a large model size. This could imply that the actual model complexity is much smaller than its apparent dimension and may be a key mechanism for understanding why highly overparametrized networks can still generalize. 
For example, Ref. ~\cite{bartlett2019nearly} showed that the VC dimension of a two-layer piecewise linear network is $O(\psi)$. Our result implies that the nonidentical number of neurons must be $O(1)$ after training. This, in turn, means that the VC dimension of the model is independent of width.



\section{Representation Learning Requires Parameter Symmetry}\label{sec: representation}

\begin{wrapfigure}{r!}{0.4\linewidth}
    \centering
    \vspace{-2em}
    \includegraphics[width=\linewidth]{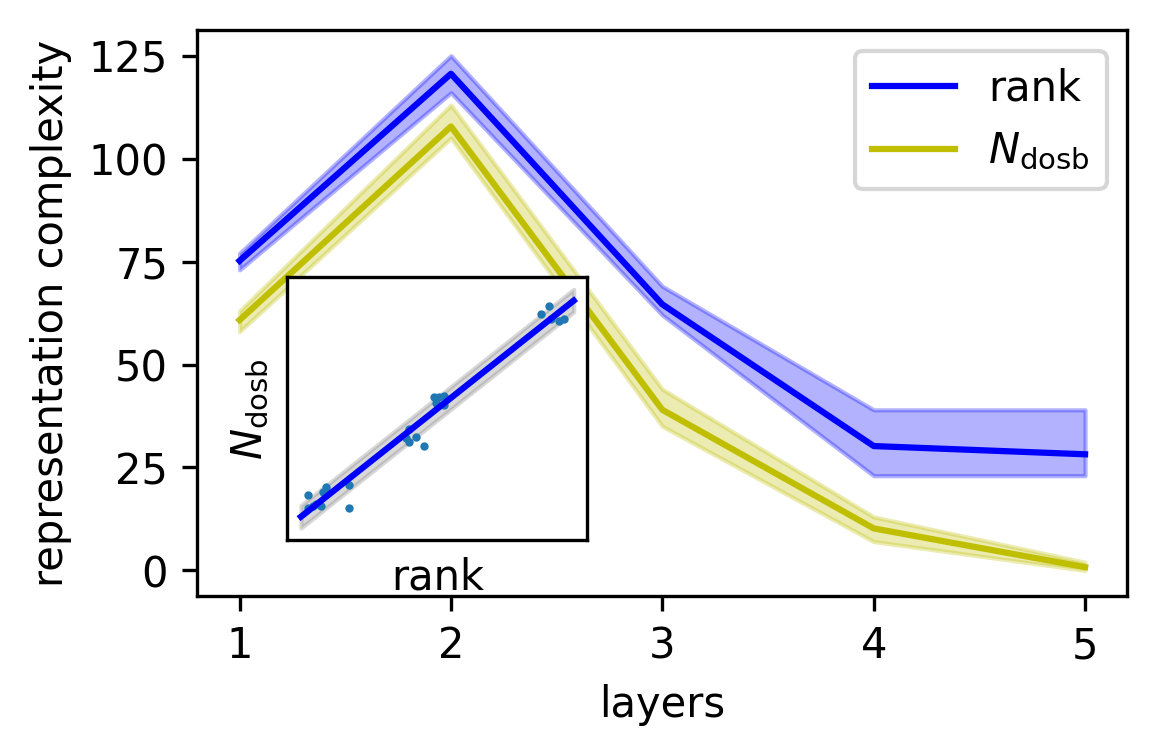}
    \caption{\small Neural networks learn a hierarchical representation. Recent works on representation learning have suggested that the rank of the latent representation first increases and then decreases through the layers \cite{xu2023janus, masarczyk2024tunnel}. This is reasonable because, on the one hand, a network needs to be wide enough to learn disconnected decision regions \cite{nguyen2018neural}, while permutation symmetries drive towards low-rankness \cite{ziyin2024symmetry}. This figure repeats the experiment in Ref. \cite{masarczyk2024tunnel} and shows the rank and degree of symmetry breaking, which can be seen as the simplest metrics of the representation complexity, of different layers in a 5-layer FCN trained on CIFAR-10. This experiment shows hierarchical representations may be due to symmetry changes: the beginning layers feature symmetry breaking, and later layers are primarily symmetry restoration.}
    \label{fig:hierachy}
    \vspace{-7em}
\end{wrapfigure}
Representation learning is believed to be the most essential aspect of deep learning \cite{bengio2013representation}. 
Learned representations of neural networks are found to take almost universally hierarchical forms, where earlier layers encode a large variety of low-level features and later layers learn a composed and abstract representation that is invariant to the changes in the low-level details \cite{zeiler2014visualizing}. One reason why symmetry may serve as a driving mechanism for learning these structured representations is that these structures almost always involve compressing information onto a few neurons and come with a low-rank structure in the hidden layer \cite{alain2016understanding, masarczyk2024tunnel, xu2023janus, papyan2020prevalence}, and a primary low-rank mechanism for deep neural networks is through the permutation symmetry of the layer or the rescaling symmetry of ReLU neurons \cite{fukumizu1996regularity, fukumizu2000local, ziyin2024symmetry}. As a primary example, neural collapses (NC) often happen in an image classification task, where the inner-class variations are found to disappear in the last and intermediate layers of the neural network \cite{papyan2020prevalence}, resulting in a representation whose rank matches the number of classes. This implies that the network has learned a hierarchical representation, where, in the later layers, only high-level features are encoded. These results motivate the following hypothesis:
\begin{boxC}
\paragraph{Representation Hypothesis:} Learning invariant, hierarchical and universal latent representations requires parameter symmetry. 
\end{boxC}
\begin{figure}[t!]
    \centering
    \includegraphics[width=0.28\linewidth]{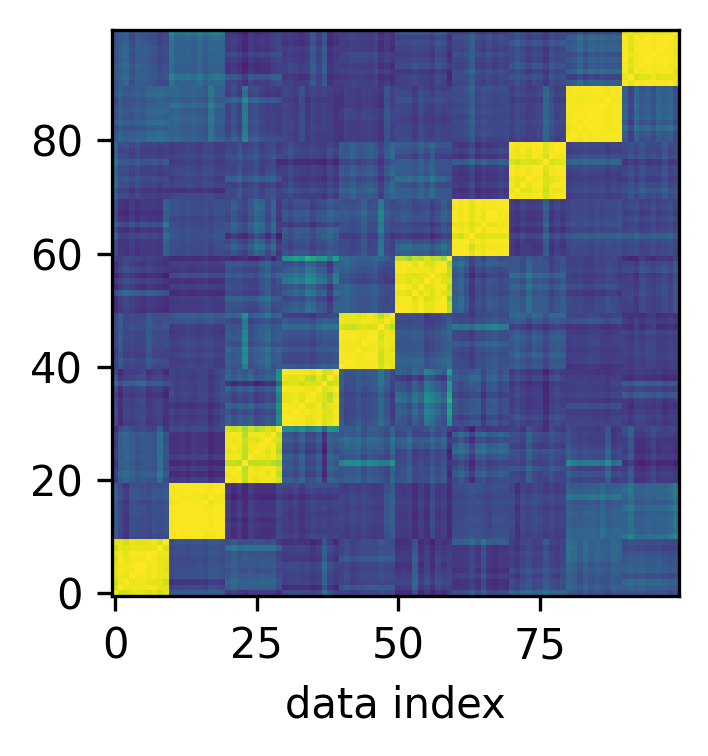}
    \includegraphics[width=0.28\linewidth]{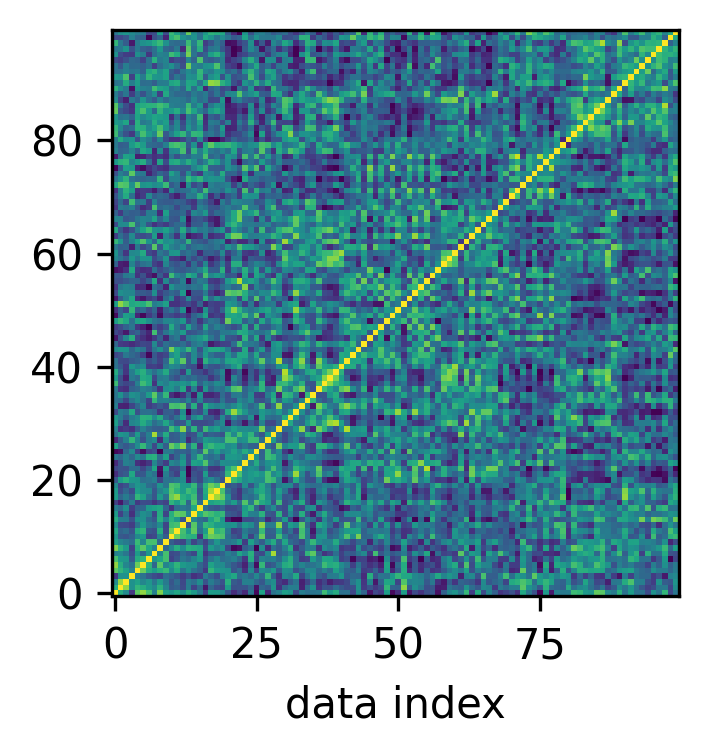}
    \caption{\small Neural collapse (NC) only happens when permutation symmetry is present. NC is a primary example of how invariant high-level representations emerge in neural networks \cite{papyan2020prevalence} and exist quite generally in classification, regression, and large language models \cite{andriopoulos2024prevalence, wu2024linguistic}. When NC happens, the learned representation must be low-rank; however, Ref.~\cite{ziyin2024remove} showed that if the permutation symmetries are removed, the learned representation is always full-rank. This implies that permutation symmetry is a necessary condition for NC to happen. The figures show the representation alignment of $100$ CIFAR10 images across $10$ classes ($10$ images in each class) and illustrates the result of Ref. ~\cite{ziyin2024remove}. The color represents the correlation between representations. \textbf{Left}: The vanilla model exhibits neural collapse, where all neurons are similar for the same class. \textbf{Right}: The innerclass variation becomes significant when the permutation symmetry is removed.} 
    \label{fig:neural collapse}
\end{figure}
\paragraph{Invariant and Hierarchical Representation Learning} See Figure~\ref{fig:neural collapse}. Here, we train a standard ResNet18 on CIFAR-10, which is known to exhibit NC. In comparison, we also train a ResNet18 whose symmetries have all been removed using the method proposed in Ref. \cite{ziyin2024remove}. We see that after removing permutation symmetries, the innerclass variation of representations no longer vanishes. In Section~\ref{app sec: neural collapse}, we show that the spectral gap between class means and innerclass variations also becomes smaller.

More broadly, it is important to understand how representations build up as the input data passes through layers of a neural network. It seems likely that there are at least three distinctive regimes in the layers of a trained net: the first few layers of neural networks serve as an \textit{expansion} phase where the representation becomes linearly separable \cite{alain2016understanding}, which requires the layer to be wide and implies a high rank \cite{nguyen2018neural}; then, a ``reduction" phase happens where the irrelevant information is thrown away and the neurons encode more and more compact information \cite{xu2023janus, rangamani2023feature}; lastly, a ``transmission" phase where the layers do nothing except transmitting the signal it receives \cite{masarczyk2024tunnel}. The NC can be seen as an example of the transmission and the simplest case of such a hierarchical representation \cite{ziyin2024formation}. These results all suggest that the representation rank is a good metric of the complexity of the layer. 
As Figure~\ref{fig:hierachy} shows, these changes in the representation ranks match well the symmetry-breaking level of the layer, a strong evidence that symmetry may drive the formation of hierarchical representations.

\paragraph{Universal Representation Learning} 
The neural collapse is also a special case of what we will call ``universal representation learning." Recent works found that the representations of learned models are found to be universally aligned to different models trained on similar datasets \cite{bansal2021revisiting, kornblith2019similarity}, and to the biological brains \cite{yamins2014performance}. This intriguing phenomenon has a philosophical undertone and has been termed ``Platonic Representation Hypothesis" \cite{huh2024platonic}. Here, we say that the two neural networks have learned a universal representation if for all $x_1,\ x_2$, 
\begin{equation}\label{eq: universal alignment}
    h_A(x_1)^T h_A(x_2) = h_B(x_1)^T h_B(x_2),
\end{equation}
where $h_A$ is the normalized activation of network $A$ in one of the hidden layers, and  $h_B$ for network B. This is an idealization of what people have observed -- and the difference between the two sides is the ``degree of alignment". The NC is a special case because when NC happens, the representations form a simplex tight frame, and so two networks with NC must have an aligned representation.

Recent results have suggested that symmetry may play a key role in shaping the representation of neural networks. The following theorem from Ref. \cite{ziyin2025neuralthermodynamicsientropic} shows that the double rotation symmetry in deep linear networks leads to the emergence of universal representations in different models, even when data and models are different.

\begin{theorem}\label{theo: alignment}
    (Informal, \cite[Theorem 8]{ziyin2025neuralthermodynamicsientropic}) Let $\mathcal{D}_M = \{(Mx_i, y_i)\}_i$ be the training set, where $M$ is an invertible matrix. Let deep linear networks A and B have arbitrary depths and widths. Let model A train with SGD on $\mathcal{D}_M$ and model B with SGD on $\mathcal{D}_{M'}$ for arbitrary $M$, $M'$. Then, if the model converges and is at the global minimum, every layer $h_A$ of A is \textbf{perfectly} aligned with every layer $h_B$ of B for any $x$:
    \begin{equation}
        h_A(x)  \propto R h_B(x),
    \end{equation}
    where $R$ is an orthogonal matrix.
\end{theorem}
See Figure~\ref{fig:representation alignment}. This is an extraordinary fact: due to the double rotation symmetry, there exist infinitely many global minima for a deep linear network such that the representations are not aligned. Yet, SGD training finds a special solution such that every layer of $A$ is aligned with every layer of $B$ despite having different initializations, different data transformations (controlled by the arbitrary matrix $M$), and different architectures (width and depth). This is only possible if the first layer transforms the representation into an input-independent form. See Figure~\ref{fig:representation alignment} for a demonstration. This theorem is a direct (perhaps the first) proof of the platonic representation hypothesis, implying that for any $x_1, x_2$, Eq.~\eqref{eq: universal alignment} holds. Importantly, the mechanism does not belong to any previously conjectured mechanisms (capacity, simplicity, multitasking \cite{huh2024platonic}). This example has nothing to do with multitasking. The result holds for any deep linear network, all having the same capacity and the same level of simplicity because all solutions parametrize the same input-output map. Here, the cause of the universal representation is symmetry alone: in the degenerate manifold of solutions, the training algorithm prefers a particular and universal one. This example showcases how symmetry is indeed an overlooked fundamental mechanism in deep learning.

\begin{wrapfigure}{r!}{0.42\linewidth}
 \centering   \includegraphics[width=\linewidth]{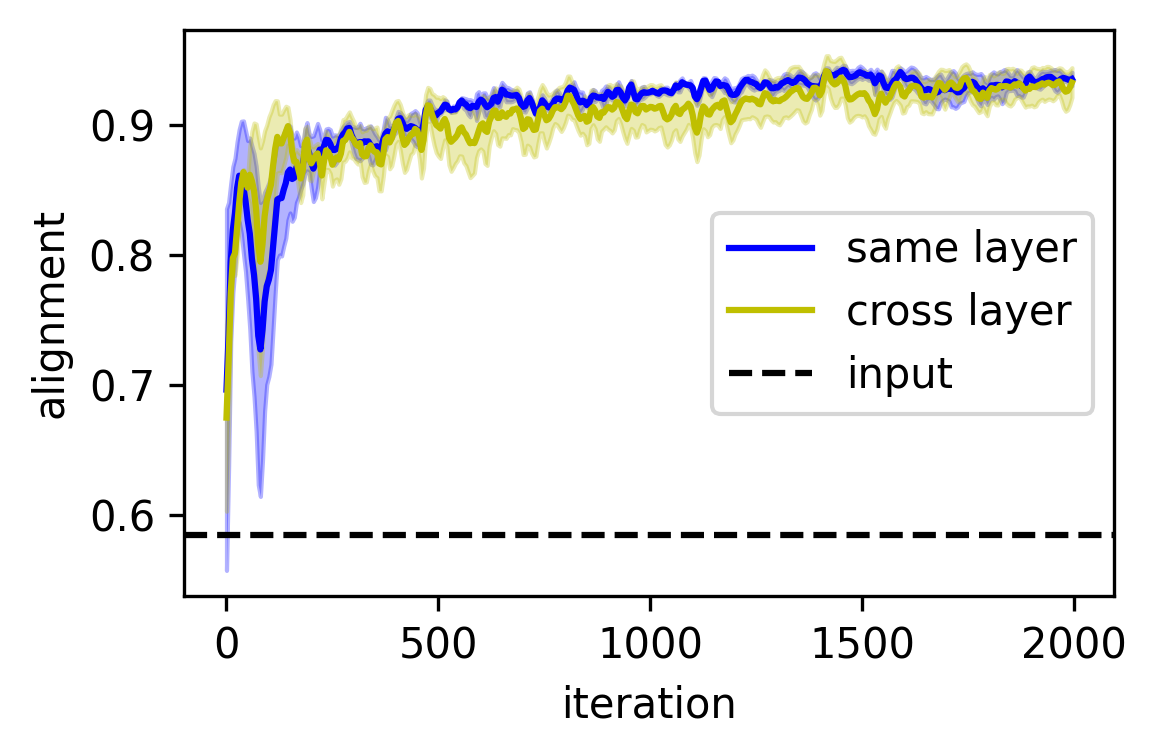}
    \caption{\small Universally aligned representations emerge in differently trained neural networks. Many recent works demonstrate that different trained neural networks learn representations that are similar to each other and even to the biological brain \cite{huh2024platonic, kornblith2019similarity, yamins2014performance}. Parameter symmetry may be a core mechanism for this intriguing phenomenon as it implies a universal alignment between representations \cite{ziyin2024parameter}. The figure is reproduced from Ref.~\cite{ziyin2025neuralthermodynamicsientropic} and shows the representations of two deep linear networks independently trained on randomly transformed MNIST become perfectly aligned for \textbf{every} pair of layers. The figure shows the average alignment between the same or different layers of two networks. This alignment does not weaken even if the input is arbitrarily transformed (Theorem~\ref{theo: alignment}). The black dashed line shows the average alignment to the input data, which is significantly weaker.}
    \label{fig:representation alignment}
    \vspace{-3em}
\end{wrapfigure}
An important future step is to generalize it to nonlinear networks. 
Due to extensive connectivity among the global minima of overparametrized networks \cite{nguyen2019connected}, there is a great potential that parameter symmetry (in exact or approximate forms) plays a similar role in overparameterized neural networks \cite{zhao2022symmetry, zhao2023understanding}.


\section{Mechanism and Control}\label{sec: mechanisms}

\paragraph{Mechanism:} Since symmetry breaking is just the lack of symmetry restoration, we may focus on restoration. A primary known mechanism for symmetry restoration is regularization in explicit or implicit forms. In explicit form, a simple weight decay has been shown to universally turn symmetric solutions energetically favorable to symmetry-broken solutions \cite{ziyin2024symmetry}, in a manner similar to phase transition in physics \cite{landau2013statistical}. In implicit form, the stochasticity in SGD is known to lead to an ``implicit" regularization effect, sometimes similar to that of weight decay \cite{kunin2021rethinking, chen2023stochastic}. Another mechanism is data augmentation, as it can also be seen as a form of regularization \cite{dao2019kernel}. While understanding these mechanisms is an open problem, the easiest way to study it is to consider the entropic loss function (also known as the ``modified loss") under SGD training \cite{geiping2021stochastic, smith2021origin, ziyin2025neuralthermodynamicsientropic}:
\begin{equation}
    L(\theta) = \underbrace{\E_\epsilon [L_0 (\theta, x_\epsilon)]}_{\text{learning + data aug.}} + \underbrace{\gamma \|\theta\|^2}_{\text{weight decay}} + \underbrace{T\, {\rm Tr}[\Sigma(\theta)]}_{\text{training noise}},
\end{equation}
where $L_0$ has the symmetry under consideration, $\Sigma$ is the gradient covariance matrix due to SGD sampling, and $T=\eta/S$ is the learning-rate-to-batch-size ratio. The data augmentation term can also be controlled by, say, a variance term $\sigma$. A careful inspection suggests that none of these terms breaks the $Z_2$ symmetries of $L_0$. So all $\sigma,\ \gamma,\ T$ can serve as regularizers that can induce a ``spontaneous symmetry breaking," which means that the symmetric states are energetically favored solutions at a critical regularization strength \cite{ziyin2024symmetry}. We experimentally demonstrate these effects in Section~\ref{app sec: mechanism}.  

With explicit regularization, whether symmetry breaking or restoration happens depends on the competition between regularization $\gamma$ and the data signal. The more curious situation is when $\gamma =0$ and $T>0$. This type of symmetry breaking is more like a thermodynamic phase transition in physics as it is induced by an entropic force due to the SGD noise. Why do noisy dynamics favor symmetric states? One intuition is that the ``temperature" of the dynamics is parameter-dependent, and stochastic dynamics tend to move to places that are ``cold," a common phenomenon in nature \cite{duhr2006molecules, anzini2019thermal}. 
In other words, temperature gradient tends to create a flow. In the case of SGD dynamics on a neural network, the symmetric solutions are `` colder" because they have a lower-rank $\Sigma$ than the symmetry-broken solutions. This discussion offers a likely intuition for the mechanism, but a more formal treatment of the argument is still lacking. Different training techniques likely favor different types of symmetries, and one may leverage symmetry to understand the inductive biases of deep learning techniques.  

\paragraph{Control:} If hierarchies and complexity adaptivity are desired, one can introduce artificial symmetries to the neural network. A straightforward way to introduce symmetries is by introducing an additional parameter $v$ and multiplying it with the parameters $W$ for which a hierarchy is desired. For example, changing a layer weight from $W$ to $vW$ introduces a new symmetry without affecting the model expressivity: $|v|=\|W\|=0$ is the symmetric state, and the model will naturally adapt to break or restore this symmetry. This is a special case of the DCS algorithm proposed in Ref. \cite{ziyin2024symmetry}. A simple way to remove symmetries is to add some static randomness to the model, which has been found to help with the model connectivity \cite{lim2024empirical} and avoid saddle points \cite{ziyin2024remove}. Finding the right symmetries to introduce or remove can be a fruitful future direction.

\section{Outlook and Alternative Views}\label{sec: conclusion}
In this position paper, we have leveraged existing results to argue that parameter symmetry is a key for unifying theories of deep learning because a wide range of seemingly unrelated hierarchies emergent in AI systems is actually related to, if not directly caused by, the symmetry of the trained models. Symmetry breaking and restoration are fundamental mechanisms in physics relevant to the dynamics of almost every scale of nature -- from scattering of the quarks to the large-scale structure formation of the universe \cite{miller1990charge, preskill1991cosmology}. If different symmetries govern the universal laws at different scales of nature \cite{anderson1972more}, it is natural to hypothesize that it may also lead to universal mechanisms and laws of learning in both artificial and biological systems. We suggest three primary hypothetical mechanisms that govern three fundamental aspects of hierarchical learning of neural networks, and validating or falsifying their related hypotheses is the most important next step. From a broader perspective, that symmetry can play such important roles in learning also calls for more interdisciplinary study of intelligence from the perspective of physics.

There are certainly alternative views to our position. 

\vspace{-4mm}
\paragraph{There is no unified explanation of deep learning:} This may very well be the case, and even laws of physics cannot explain every observation in nature. However, it is still highly possible that a small set of laws is sufficient to explain a large fraction of observations. For example, it is well possible that there may exist ten minimal laws with which we can explain 95\% of the deep learning phenomenon, and if that is the case, symmetry may be one of those laws.

\vspace{-4mm}
\paragraph{Symmetry alone is insufficient:} 
In most of our examples, it is a combination of symmetry and some other effect that leads to the phenomenon. For example, the compression bias of SGD is a result of symmetry \textit{and} training noise. Neural collapse is due to symmetry \textit{and} regularization \cite{rangamani2022neural}. 
Therefore, it is possibly the case that symmetry \textit{plus} some form of explicit (weight decay) or implicit (noise, GD training, data augmentation) regularization is the dominating factor, and this may be an important direction of future research.

\vspace{-4mm}
\paragraph{Symmetry may not be necessary:} While symmetry is possibly a unified perspective to view many phenomena, there are cases where a subset of these phenomena are exhibited but does not require symmetry. For theoretical purposes, this suggests that there can be other concepts as important as symmetry. However, from a design and practice perspective, knowing only one way to achieve a design goal often suffices -- and the parameter symmetry may be sufficient. For example, there may be multiple ways for neural collapse to happen, but to introduce neural collapse to the desired model, one only needs to know one such way.


\bibliographystyle{plain}

\appendix

\section{Experiments}
\subsection{Measurement of $\Delta^G$}\label{app sec: measurement}
\paragraph{Permutation Symmetry} For the permutation symmetry in fully connected layers, we have described the $\Delta^G$ for these pairwise symmetries. However, for a layer of width $\psi$, there are $O(\psi^2)$ many such pairs, but we do not have to care about every pair of these because, for most of our purposes, we only care about how many neurons are actually functional and how many neurons are useless. This fact can allow us to reduce the number of measurements to $\psi$. To achieve this, we first sort all the neurons according to the norm of the input and outgoing weights. Because for two neurons to become close, their norms also need to be close. Therefore, if two neurons have a large norm difference, their permutation symmetry must have been broken. Under this ordering, we measure the $\Delta^G$ for the pairs of neurons with the closest norms.

An alternative perspective to look at this is that these pairwise neuron distances form a generating set of the symmetric group, and we are counting the number of symmetry breaking of these subgroups generated by each generator.

\paragraph{Omission of G} We will write $\Delta^G$ as $\Delta$ throughout the appendix because we are often comparing $\Delta$ for different groups and so the superscript $G$ is different for most cases.

\paragraph{Double Rotation Symmetry} Computing the degree of symmetry for the double rotation symmetry is rather tricky. 
In the ViT experiment (Figure~\ref{fig:complexity upper bound}), we measured the degree of symmetry in the self-attention layers. Here, the symmetry is the double rotation symmetry:
\begin{equation}
    W_Q W_K = W_Q M M^{-1} W_K,
\end{equation}
for an arbitrary invertible matrix $M$. The symmetric states are the ones where $W_Q$ and $W_K$ both become low-rank in the same subspace. Namely, it happens when there exists a vector $n$ such that 
\begin{equation}
    W_Q n =0,\ W_K^T n =0.
\end{equation}
There are at most $k$ such $n$ where $k$ is the right dimension of $W_Q$. This motivates this group's following operational definition of $\Delta$. We first compute the eigenvalue decomposition of $W_Q^T W_Q$ to obtain its eigenvectors $U$:
\begin{equation}
    W_Q^T W_Q = U \Lambda_Q U^T.
\end{equation}
We then compute the following matrix:
\begin{equation}
    \Lambda_K =  U^T W_K W_K^T U,
\end{equation}
and the degree of symmetry is defined as 
\begin{equation}
    N_{\rm dos} = \sum_i^k \mathbbm{1}_{\Delta_i < \Delta_{\rm th}} ,
\end{equation}
where 
\begin{equation}
    \Delta_i = |(\Lambda_Q)_{ii} - (\Lambda_K)_{ii}|.
\end{equation}

\paragraph{Degree of Symmetry} The \textit{degree of symmetry} we compute is the number of all  $\Delta^G$ such that $\Delta^G> \Delta^G_{\rm th}$, where $\Delta^G_{\rm th}$ is a certain threshold, often between $0.05$ and $0.2$. If this layer has $h$ many neurons, we define
\begin{equation}
    N_{\rm dosb} = h - N_{\rm dos}\,,
\end{equation}
and by definition $N_{\rm dosb} \ge 0$.

\subsection{Learning Dynamics}\label{app sec: exp detail}
\begin{figure}
    \centering
    \includegraphics[width=0.35\linewidth]{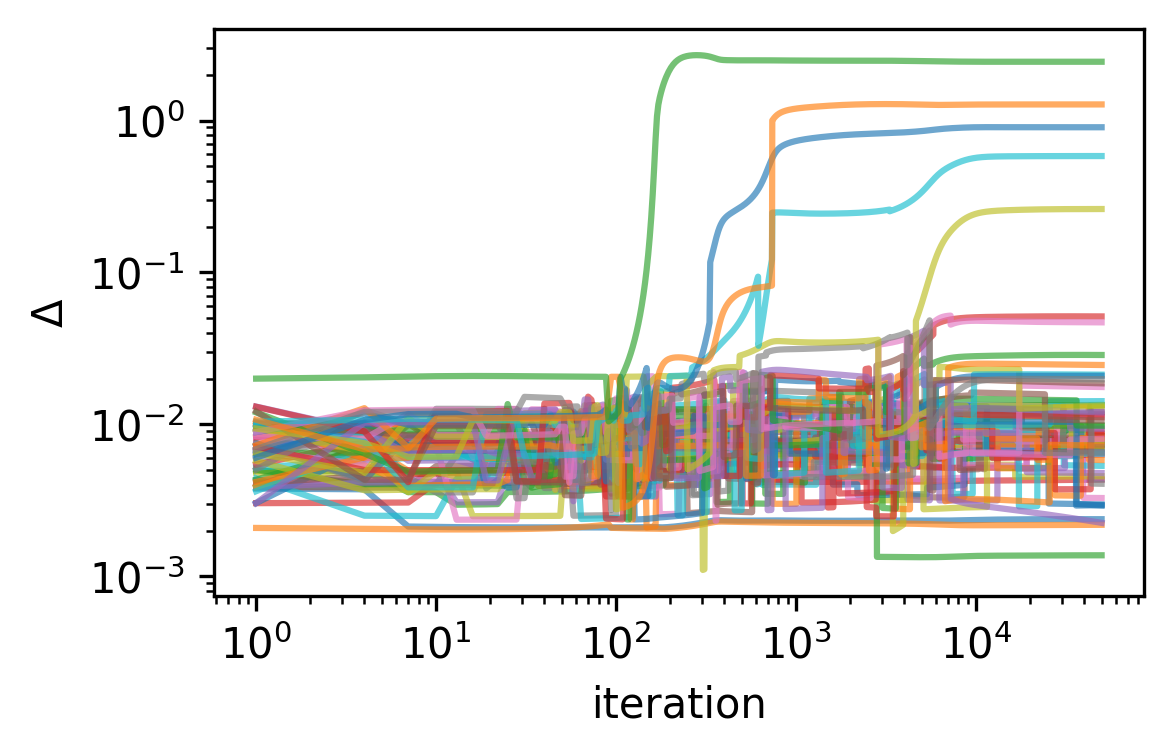}
    \includegraphics[width=0.35\linewidth]{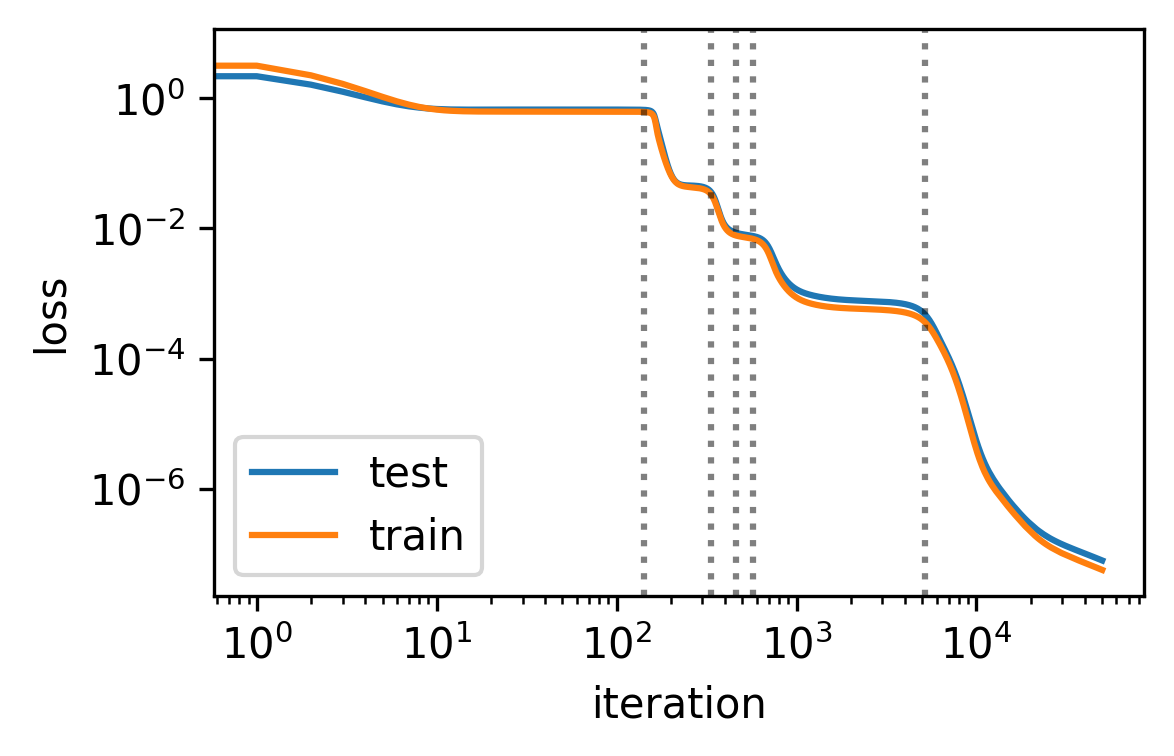}
    \caption{The same setting as Figure \ref{fig:symmetry breaking dynamics}, but for quadratic activation. \textbf{Left}: $\Delta$ during training. \textbf{Right}: loss curve during training. Black dotted lines represent the time when $\Delta$ exceeds $\Delta^G_{\rm th}$.}
    \label{fig:symmetry breaking dynamics2}
\end{figure}
\begin{figure}
    \centering
    \includegraphics[width=0.35\linewidth]{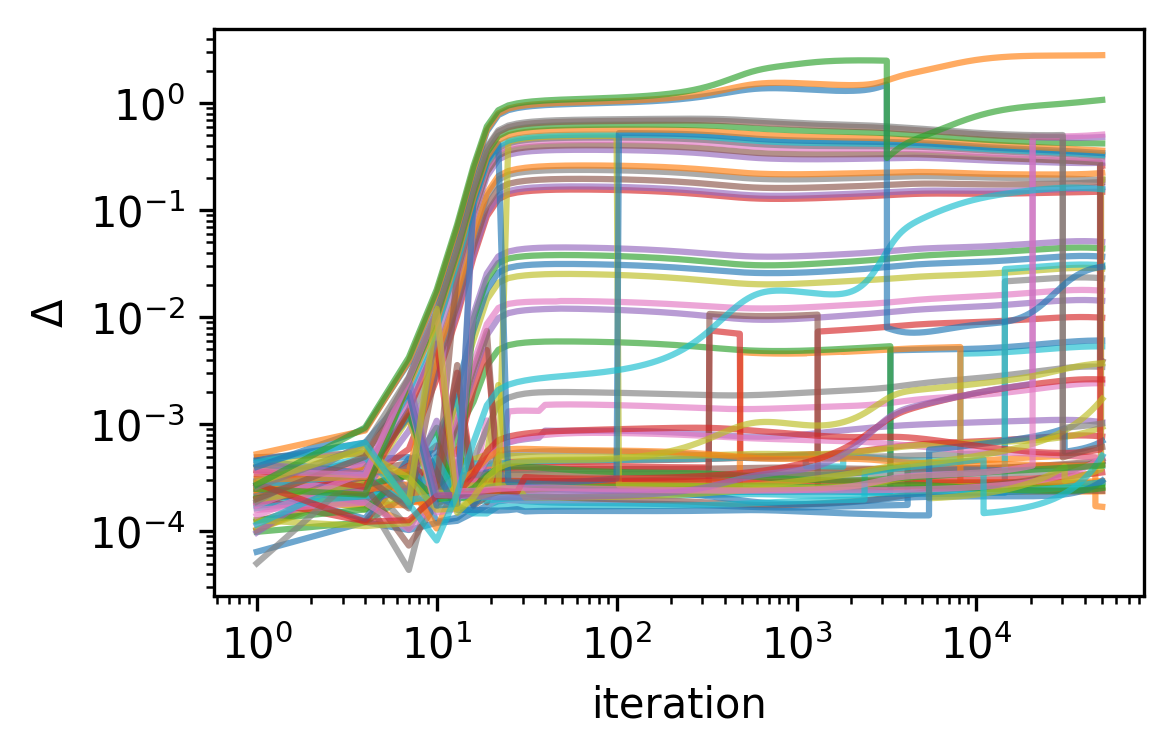}
    \includegraphics[width=0.35\linewidth]{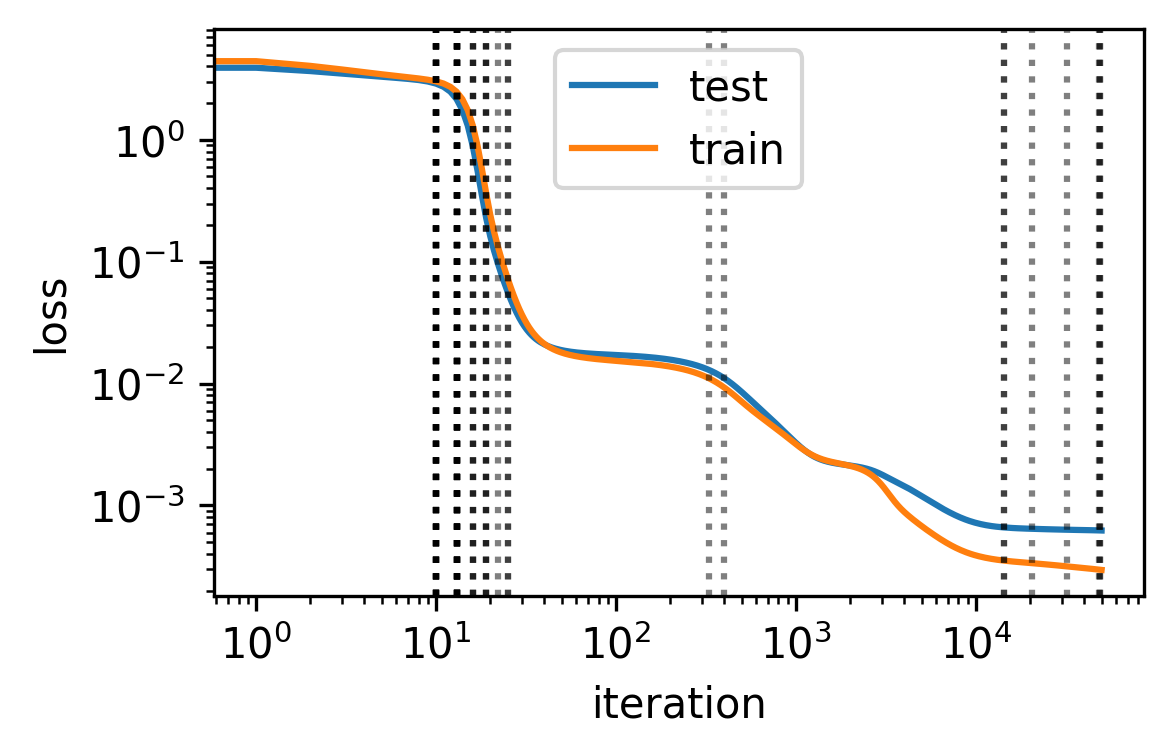}
    \caption{The same setting as Figure \ref{fig:symmetry breaking dynamics}, but for sin activation.}
    \label{fig:symmetry breaking dynamics3}
\end{figure}

\begin{figure}[t!]
    \centering
    \includegraphics[width=0.35\linewidth]{plot/rank.png}
    \includegraphics[width=0.35\linewidth]{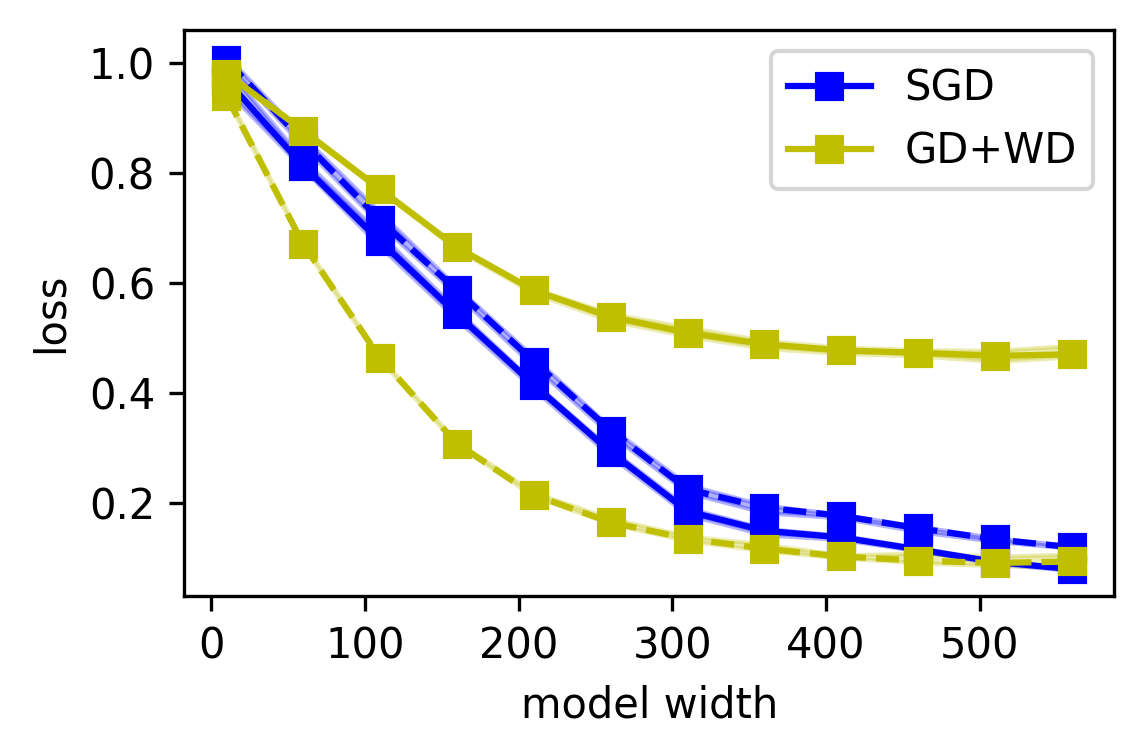}
    \caption{\small The same as Figure \ref{fig:complexity upper bound}, and we report the rank of the first layer for completeness.}
    \label{fig:complexity upper bound2}
\end{figure}

\begin{figure}[t!]
    \centering
    \includegraphics[width=0.3\linewidth]{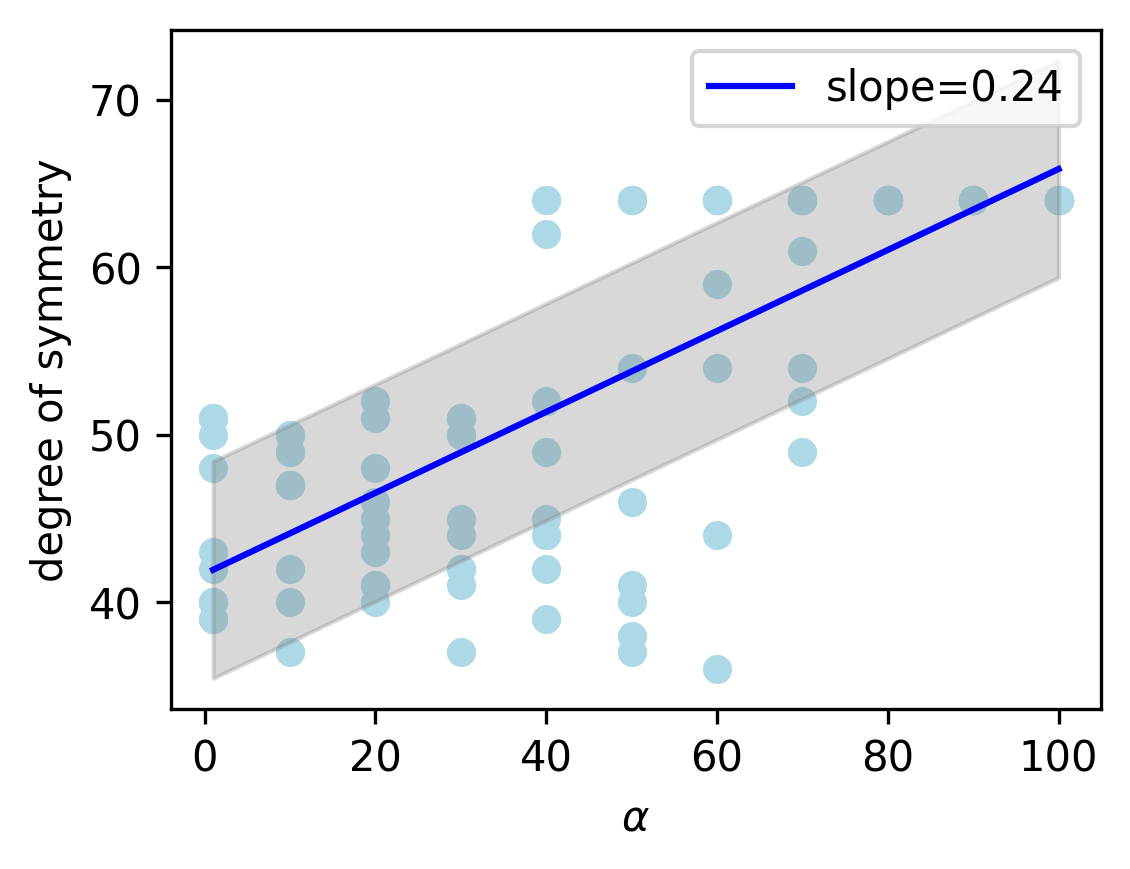}
    \includegraphics[width=0.3\linewidth]{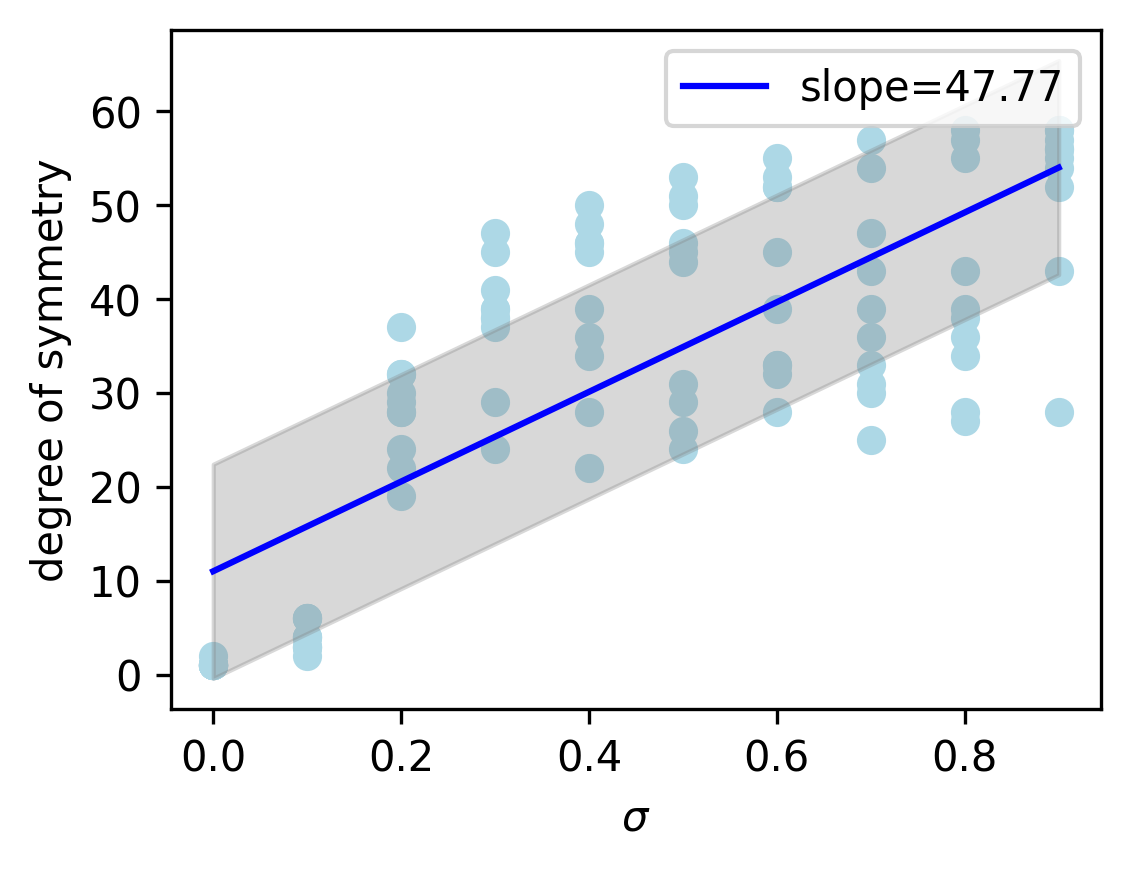}
    \includegraphics[width=0.3\linewidth]{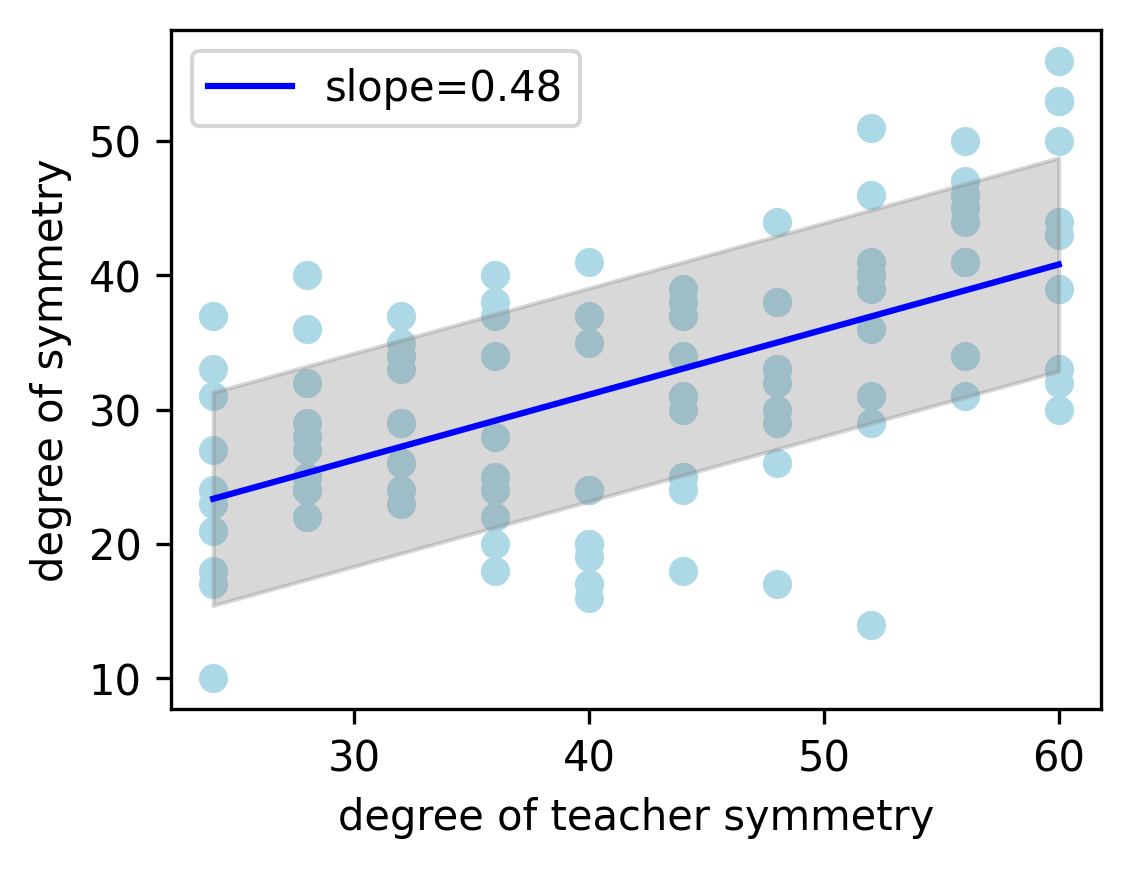}

    \includegraphics[width=0.3\linewidth]{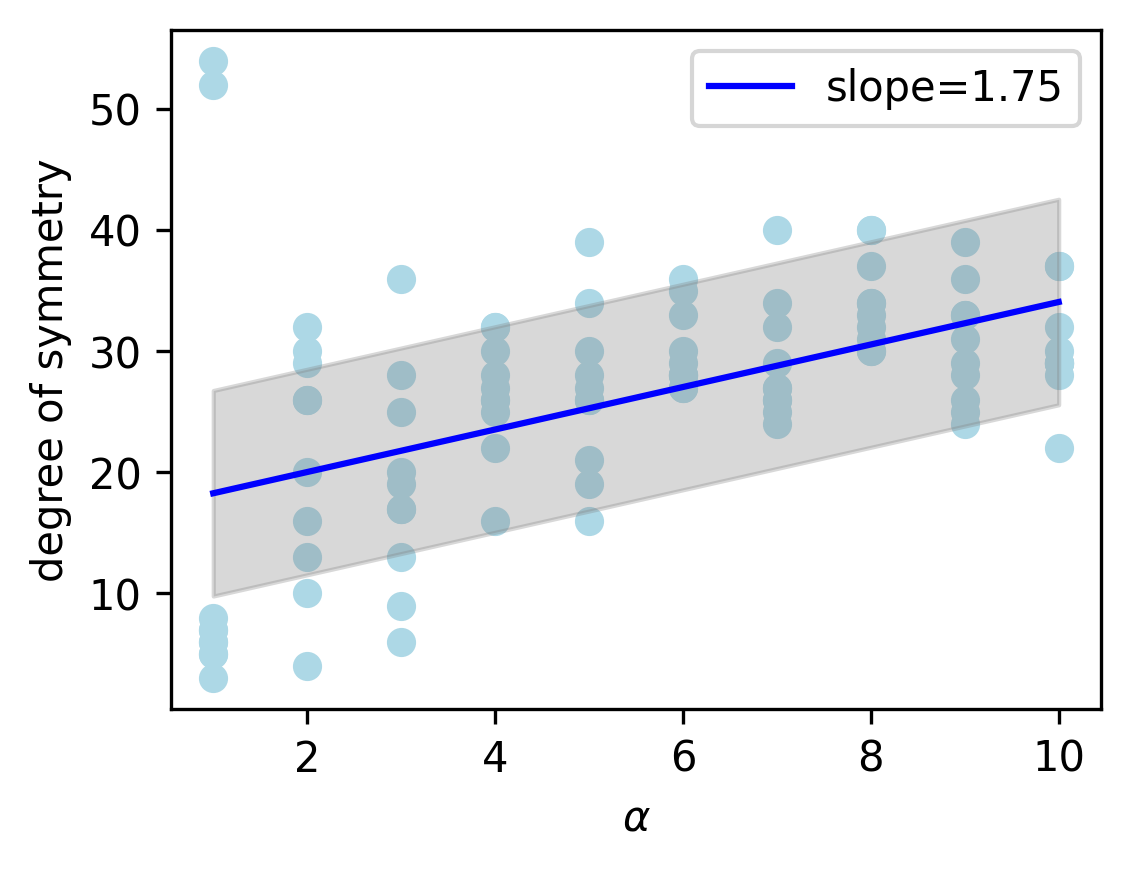}
    \includegraphics[width=0.3\linewidth]{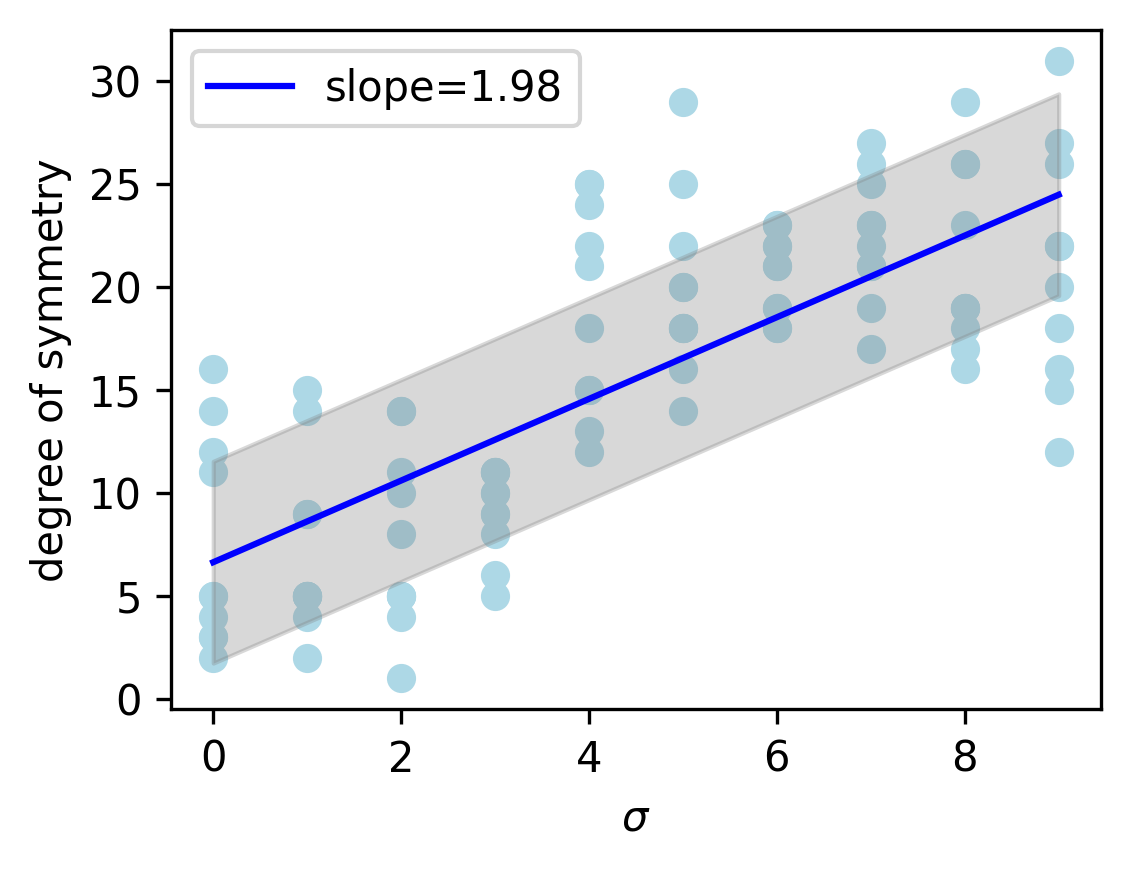}
    \includegraphics[width=0.3\linewidth]{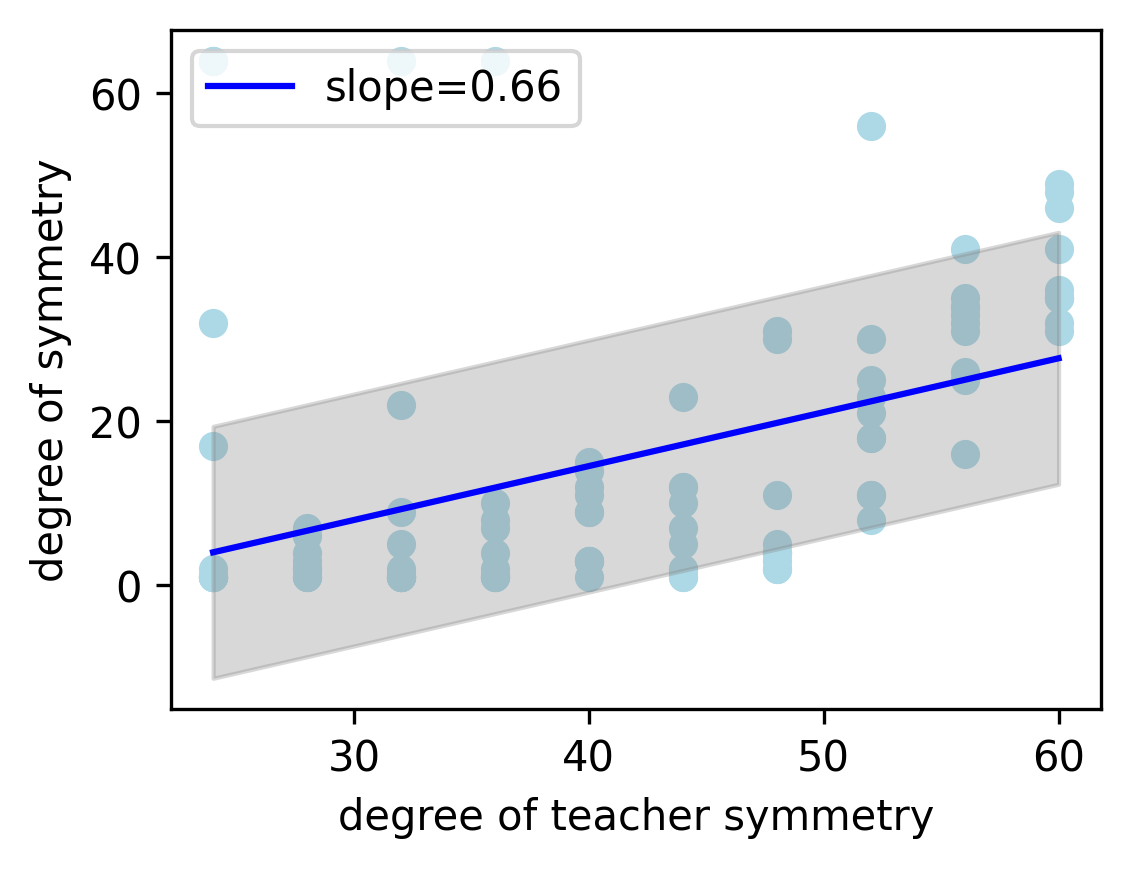}
    \caption{\small The degree of symmetry increases with $\alpha$, the inverse scaling of input (\textbf{Left}), $\sigma$, the noise level (\textbf{Middle}), and the degree of teacher symmetry ($64-$the number of teacher units, \textbf{Right}). Blue dots represent experiment results, and blue lines represent linear fitting. \textbf{Upper}: MLP. \textbf{Lower}: transformer.}
    \label{fig:degree_symmetry}
\end{figure}

In Figure \ref{fig:symmetry breaking dynamics}, we use a teacher-student setting, where both the teacher and student are five-layer fully connected networks (FCNs) with 64 units per layer and tanh activation. The input and output dimensions are $10$ and $1$, respectively. Teacher network weights follow Kaiming initialization, and the output is scaled by 10. The student network is initialized with different scales. The training dataset consists of $300$ samples generated as 
$\sin(\mathcal{N}(0,\sigma^2))$, with $\sigma$ ranging from $0$ to $3$. The test dataset contains 500 samples from the same distribution. We train the student network with gradient descent (no weight decay) for $50,000$ iterations using the mean squared error (MSE) loss.

To calculate the symmetry-breaking distance, we sort the first-layer neurons by their norms and compute the $L_2$ distance between consecutive neurons. Results for quadratic and sin activations are presented in Figures \ref{fig:symmetry breaking dynamics2} and \ref{fig:symmetry breaking dynamics3}, including distances for $63$ neuron pairs in the first layer (see Section~\ref{app sec: measurement}).

\subsection{Mechanisms for Symmetry Changes}\label{app sec: mechanism}

For the left panel of Figure \ref{fig:complexity upper bound}, we directly use the pretrained models of ViT-Base and ViT-Large from \url{https://pytorch.org/vision/stable/models.html}. The computation of the weight rank and $N_{\rm dosf}$ follows the outline in Section~\ref{app sec: measurement}.

In the middle panel of Figure \ref{fig:complexity upper bound}, we train a three-layer fully connected network (FCN) with swish activation on a synthetic dataset where inputs and outputs are identical, sampled from a 300-dimensional standard Gaussian distribution. For gradient descent (GD), the dataset size is $1000$, and weight decay is set to $10^{-3}$. For stochastic gradient descent (SGD), the batch size is $128$, with new data randomly generated for each batch. In both cases, the test dataset contains $1000$ samples, and the networks are trained using the Adam optimizer. For completeness, we report the rank of the first layer in Figure \ref{fig:complexity upper bound2}, which shows that both the rank and the generalization error remain constant when the model width gets larger.

In Figure \ref{fig:degree_symmetry}, we analyze how symmetry evolves with input scaling, label noise, and the teacher symmetry. We adopt the teacher-student framework from Figure \ref{fig:symmetry breaking dynamics}, using Gaussian inputs for the teacher network and the Adam optimizer with a weight decay of $10^{-2}$. Symmetry is quantified as the number of first-layer neuron pair distances below $0.1$ among $63$ pairs in the student network. Teacher symmetry is measured as $64-$the number of neurons in each teacher network layer. As expected, stronger noise, weaker inputs, or teachers with higher symmetry (simpler target functions) result in greater symmetry in the student network.

The second row of Figure \ref{fig:degree_symmetry} presents experiments on a simple transformer comprising two self-attention layers and one linear layer. We use an in-context learning dataset with input sequences of length $50$, where each sequence element is a $21$-dimensional token. The last dimension of each token is a linear combination of its first $20$ dimensions, sampled from a standard Gaussian. The label corresponds to the last dimension of the final token in the sequence. The transformer is trained with the AdamW optimizer (weight decay $10^{-2}$). The results for the transformer align with those for the MLP.

\subsection{Representation Learning}\label{app sec: neural collapse}
\begin{figure}[t!]
    \centering
    \includegraphics[width=0.4\linewidth]{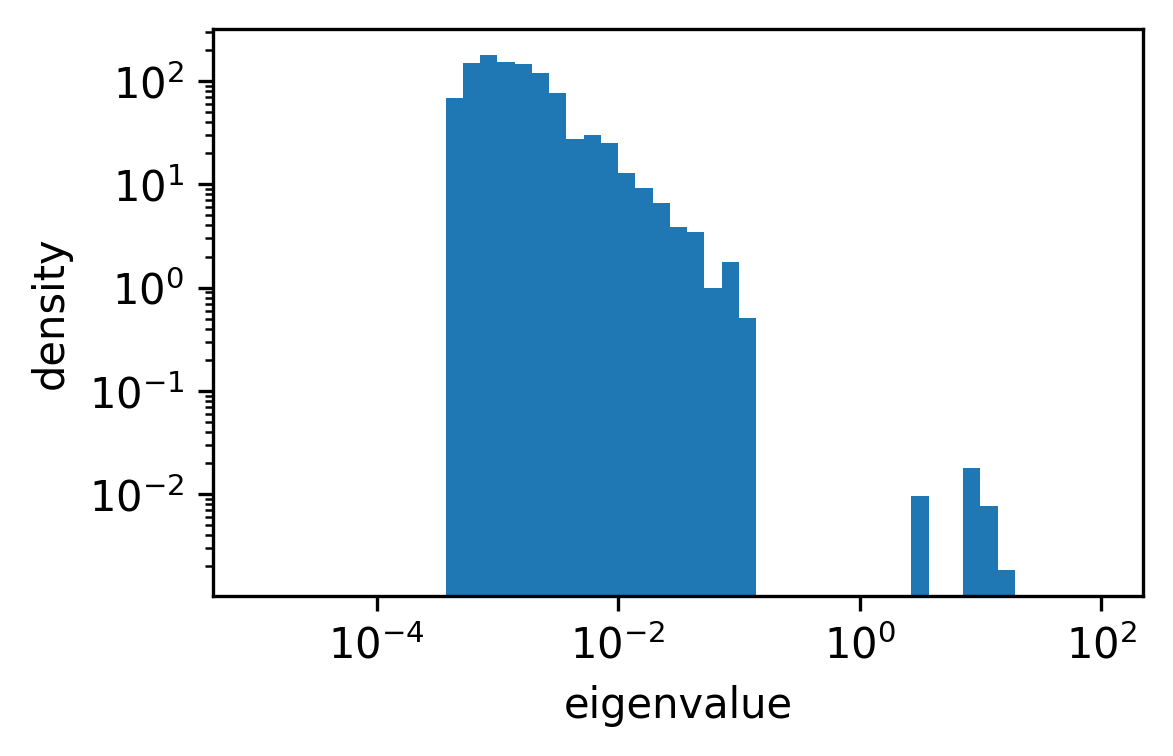}
    \includegraphics[width=0.4\linewidth]{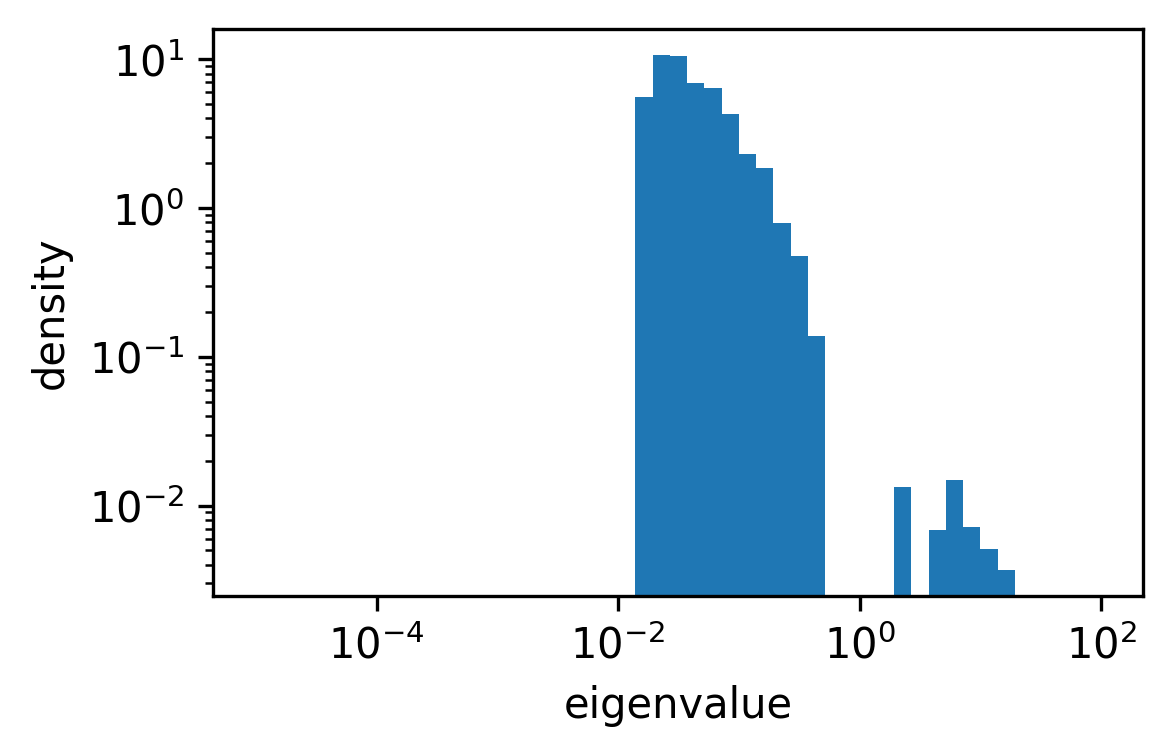}
    \caption{The spectrum of the representation correlation. \textbf{Left}: the vanilla model. \textbf{Right}: the syre model. While both models are low rank, the gaps between eigenvalues are smaller for the syre model.}
    \label{fig:neural collapse_eigs}
\end{figure}

\begin{figure}[t!]
    \centering
    \includegraphics[width=0.35\linewidth]{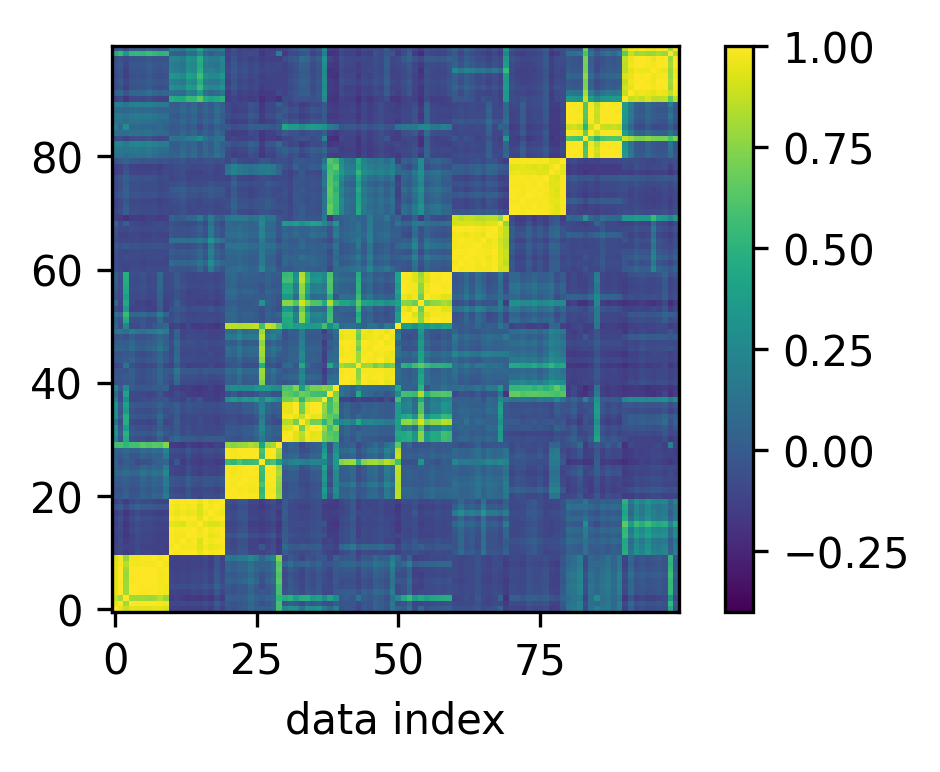}
    \includegraphics[width=0.35\linewidth]{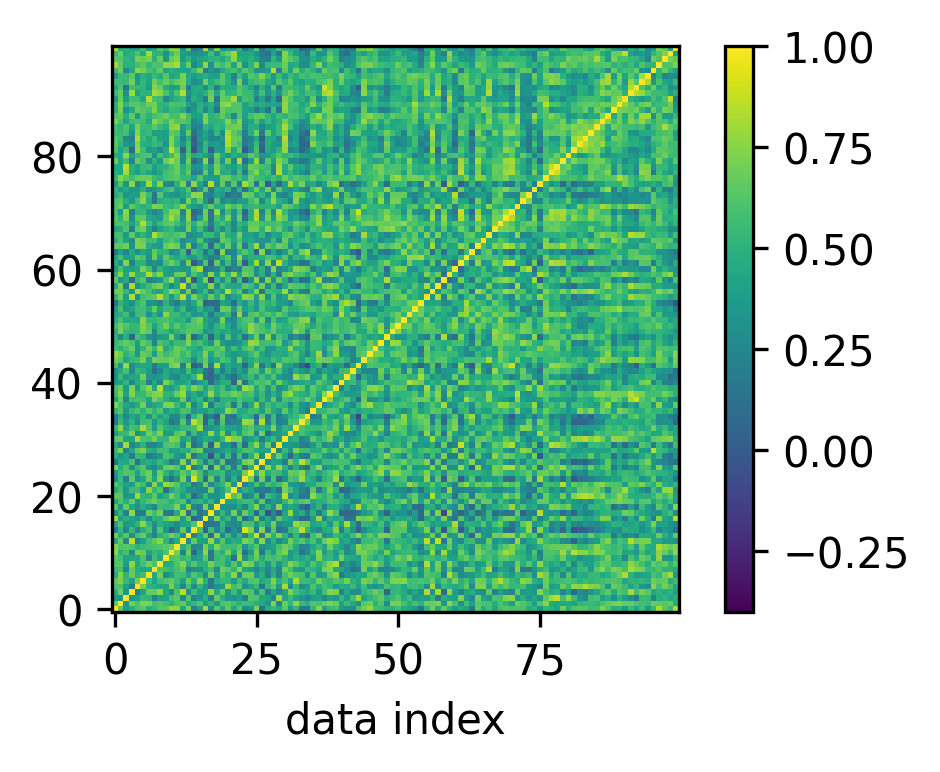}
    \caption{The same setting as Figure \ref{fig:neural collapse}, but with weight decay ten times larger.}
    \label{fig:neural collapse2}
\end{figure}

\begin{figure}[t!]
    \centering
    \includegraphics[width=0.4\linewidth]{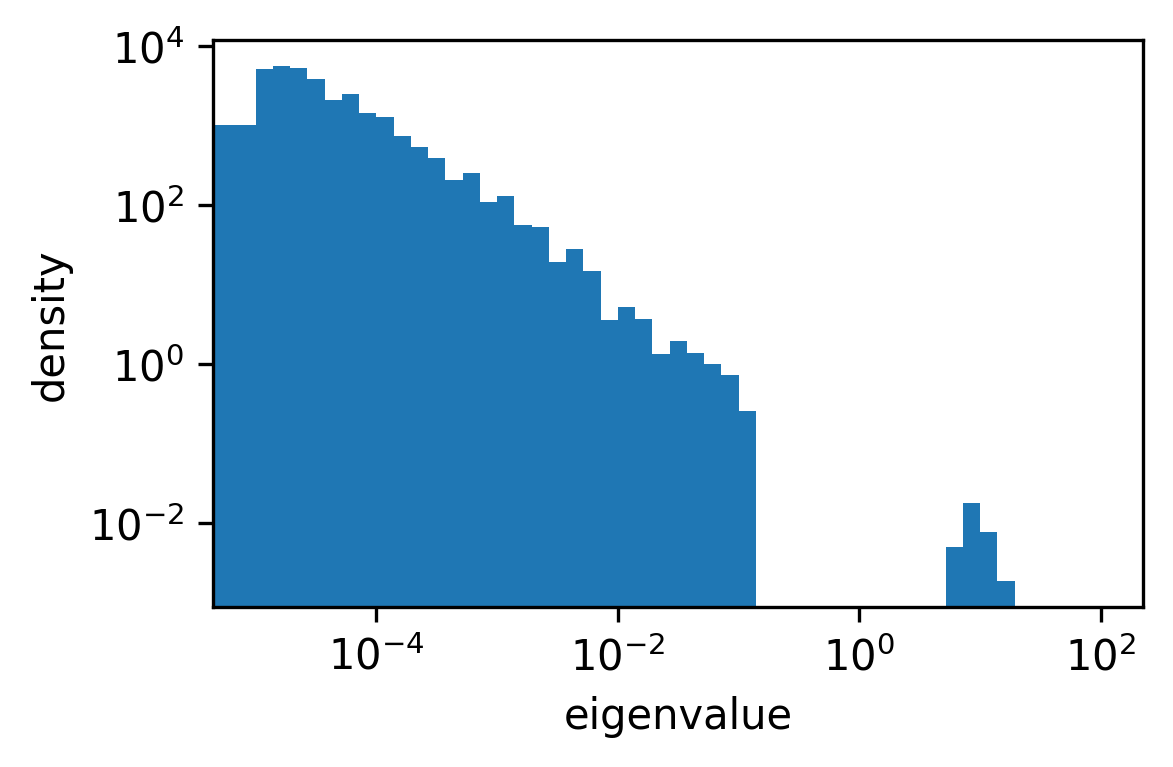}
    \includegraphics[width=0.4\linewidth]{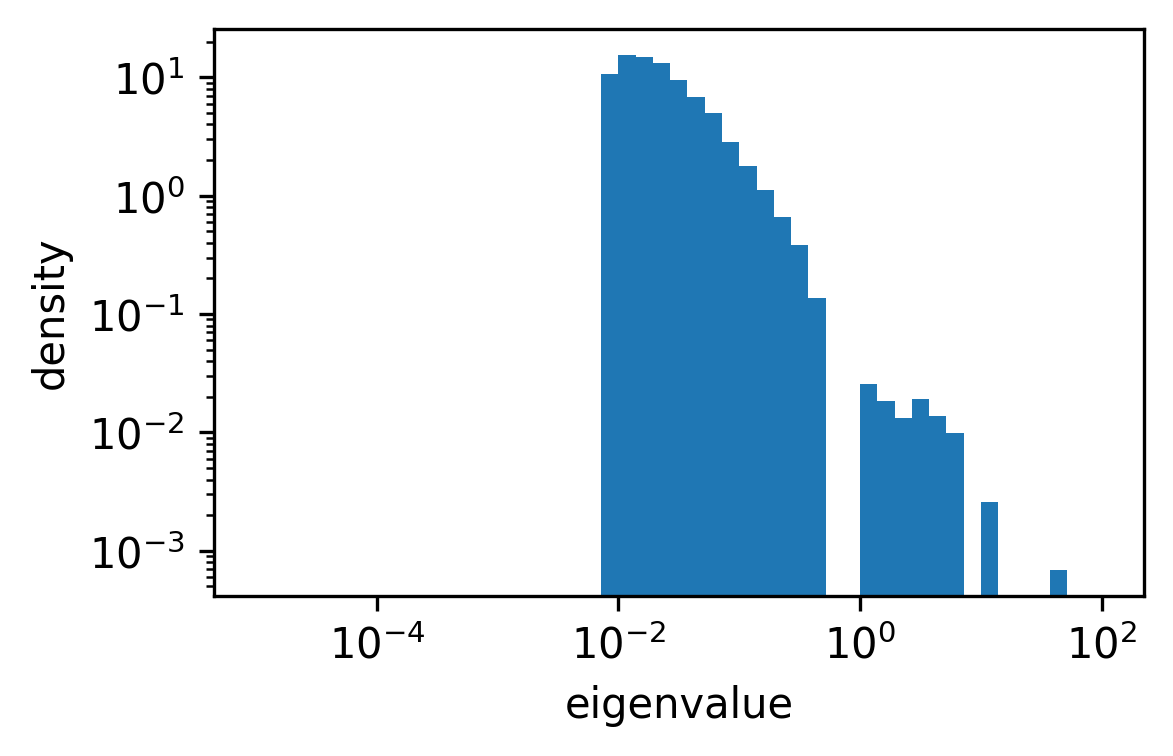}
    \caption{The same setting as Figure \ref{fig:neural collapse_eigs}, but with weight decay ten times larger. The low-rank structure is largely suppressed.}
    \label{fig:neural collapse_eigs2}
\end{figure}

\begin{figure}[t!]
    \centering
    \includegraphics[width=0.37\linewidth]{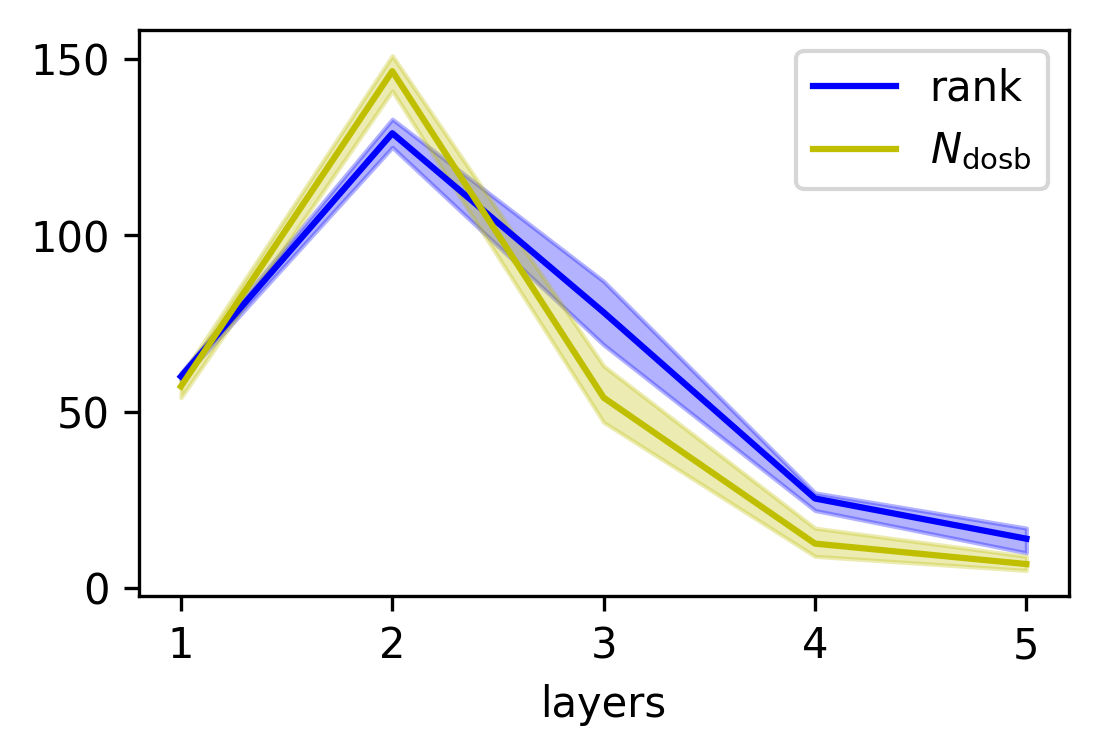}
    \includegraphics[width=0.4\linewidth]{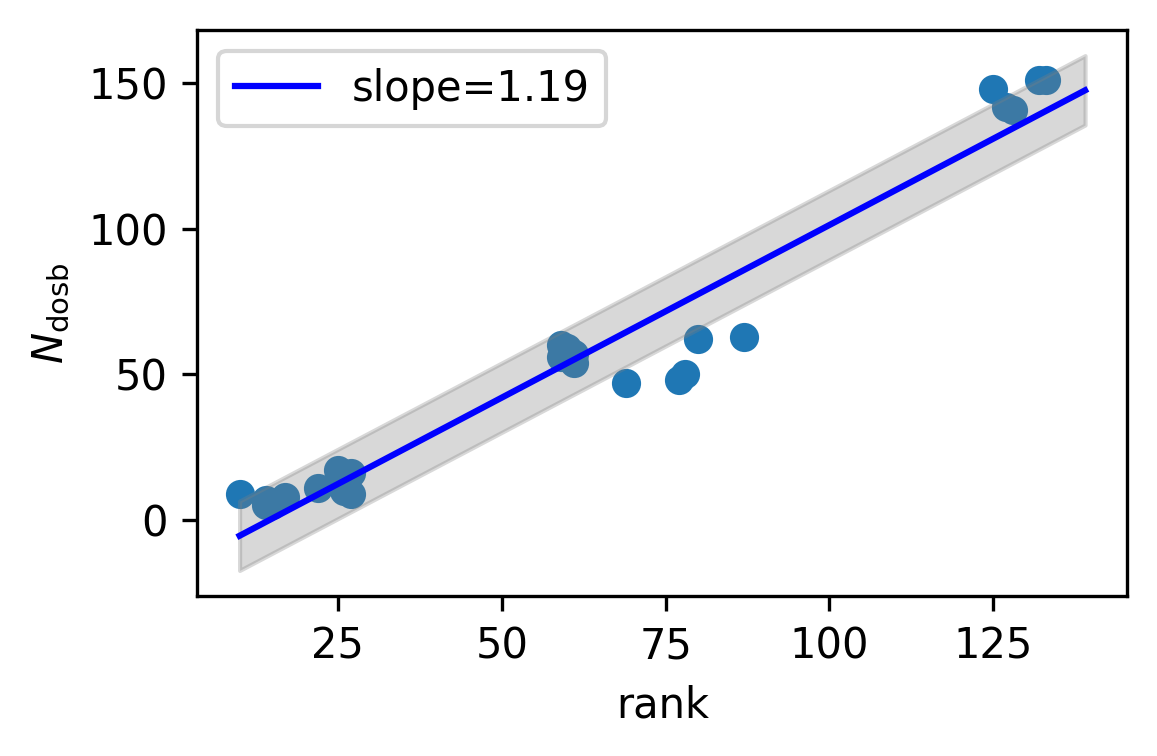}
    \caption{\small The same setting as Figure \ref{fig:hierachy}, but on the SVHN dataset.}
    \label{fig:hierachy2}
\end{figure}

\begin{figure}[t!]
    \centering
    \includegraphics[width=0.4\linewidth]{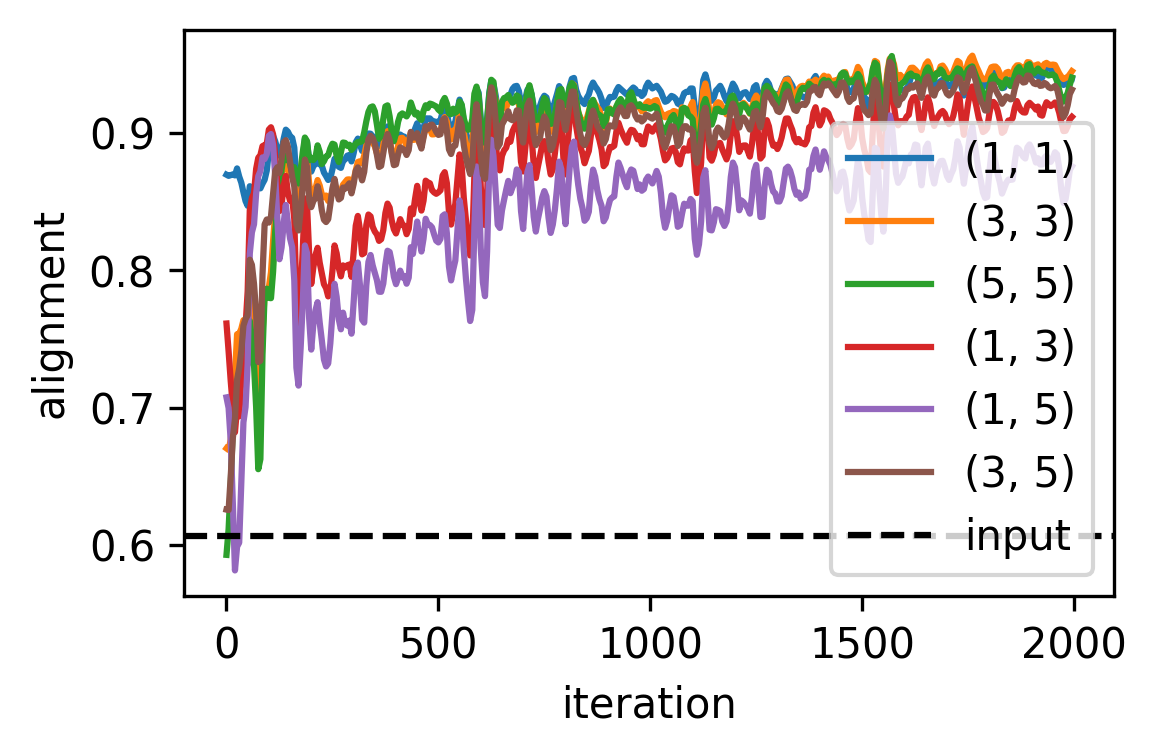}
    \includegraphics[width=0.4\linewidth]{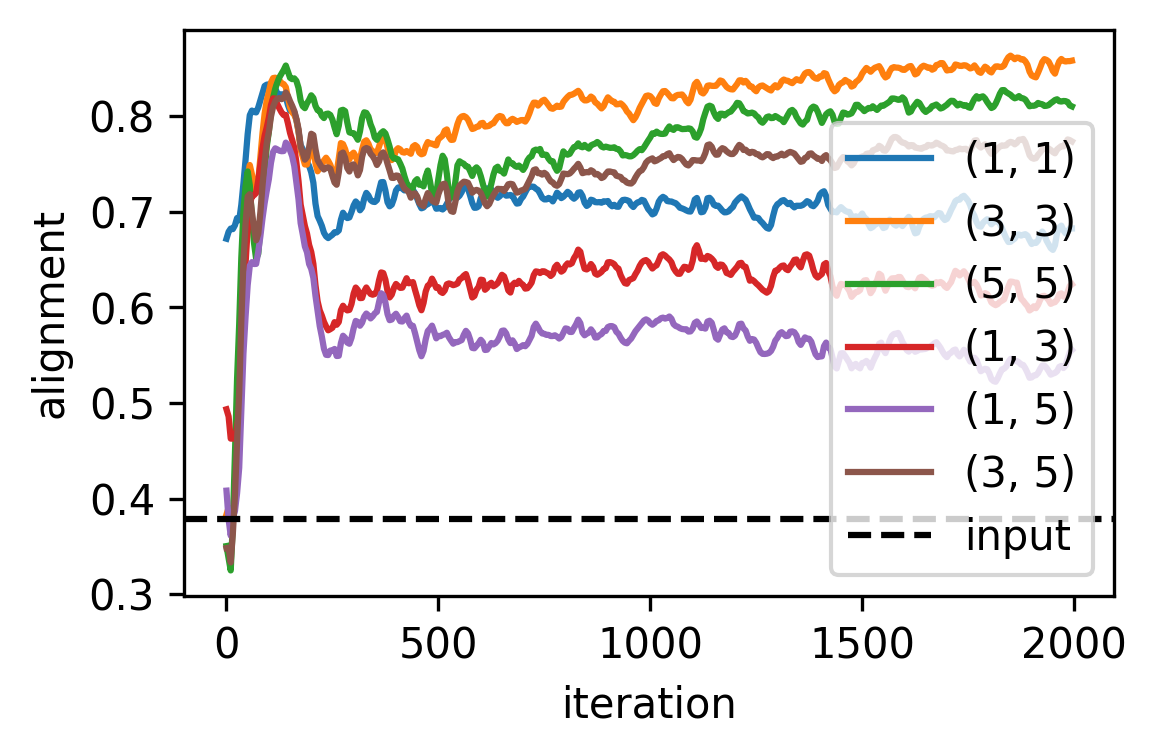}
    \caption{\small An extension of Figure \ref{fig:representation alignment}, where two networks have different number of neurons. The legend $(i,j)$ denotes the alignment between the $i$-th layer of network A and the $j$-th layer of network B. \textbf{Left}: linear network. \textbf{Right}: tanh network.}
    \label{fig:representation alignment2}
\end{figure}

In Figure \ref{fig:neural collapse}, we train ResNet18 on the CIFAR-10 dataset using vanilla weight decay (loss $\ell(\theta)+\gamma||\theta||^2$) and syre (loss $\ell(\theta)+\gamma||\theta-\theta_0||^2$) \cite{ziyin2024symmetry}. In both cases, the weight decay $\gamma$ is set to $5\times10^{-4}$. Networks are trained for $200$ epochs, and results from the final epoch are reported. To generate Figure \ref{fig:neural collapse}, we randomly select 10 images per class and compute the correlation between their features (i.e., the input to the final layer). The corresponding eigenvalue spectrum is presented in Figure \ref{fig:neural collapse_eigs}, showing 10 prominent eigenvalues. Figures \ref{fig:neural collapse2} and \ref{fig:neural collapse_eigs2} extend these experiments with a weight decay of $5\times10^{-3}$, where the vanilla weight decay model exhibits a stronger low-rank structure compared to the syre model.

In Figure \ref{fig:hierachy}, we train a five-layer FCN with $512$ neurons per layer and swish activation on CIFAR-10. Symmetry breaking is evaluated as the number of pairwise distances (as in Figure \ref{fig:symmetry breaking dynamics}) exceeding $1$. Results are averaged over 5 independent runs. Figure \ref{fig:hierachy2} replicates this experiment on the SVHN dataset, where the same hierarchical representation effects can be observed.

In Figure \ref{fig:representation alignment}, we train two deep linear networks on MNIST using the Adam optimizer, where the input of the second one is randomly transformed by an invertible matrix, , whose elements are sampled from the standard Gaussian distribution. Each network is a six-layer fully connected network (FCN) with 128 neurons per hidden layer. The average alignment to input data is computed as the average alignment between the data and the third layer during training. Figure \ref{fig:representation alignment2} explores a similar setting for both linear and tanh networks, and the two networks have hidden layers with $64$ and $128$ neurons, respectively, where the universal alignment effects still hold.

\clearpage
\subsection{Adaptive Capacity}\label{app sec: adaptive capacity}
\begin{figure}[t!]
    \centering
    \includegraphics[width=0.245\linewidth]{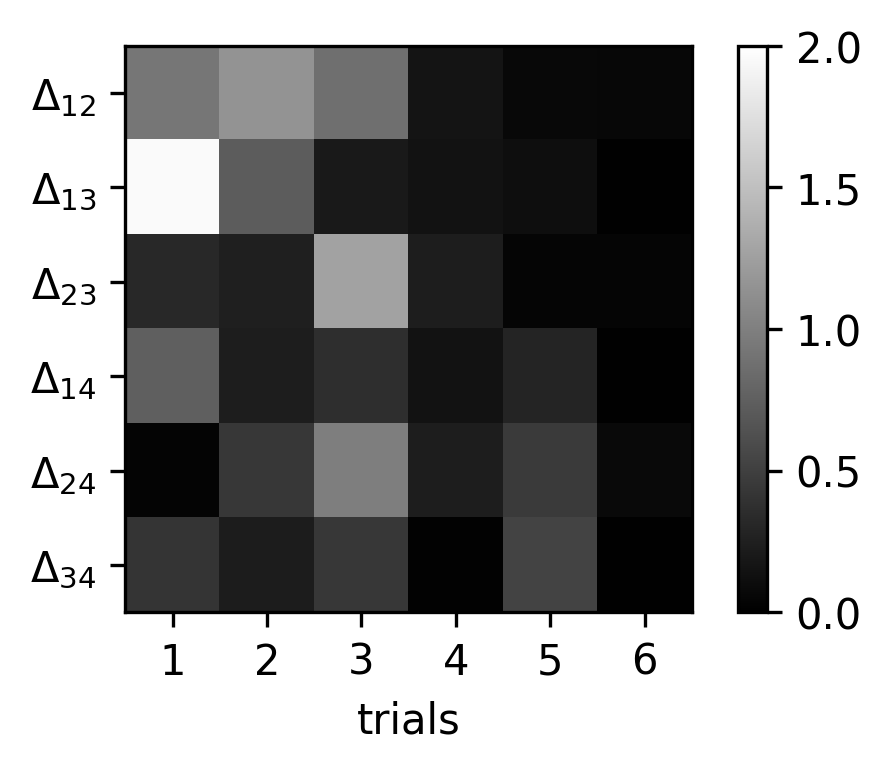}
    \includegraphics[width=0.245\linewidth]{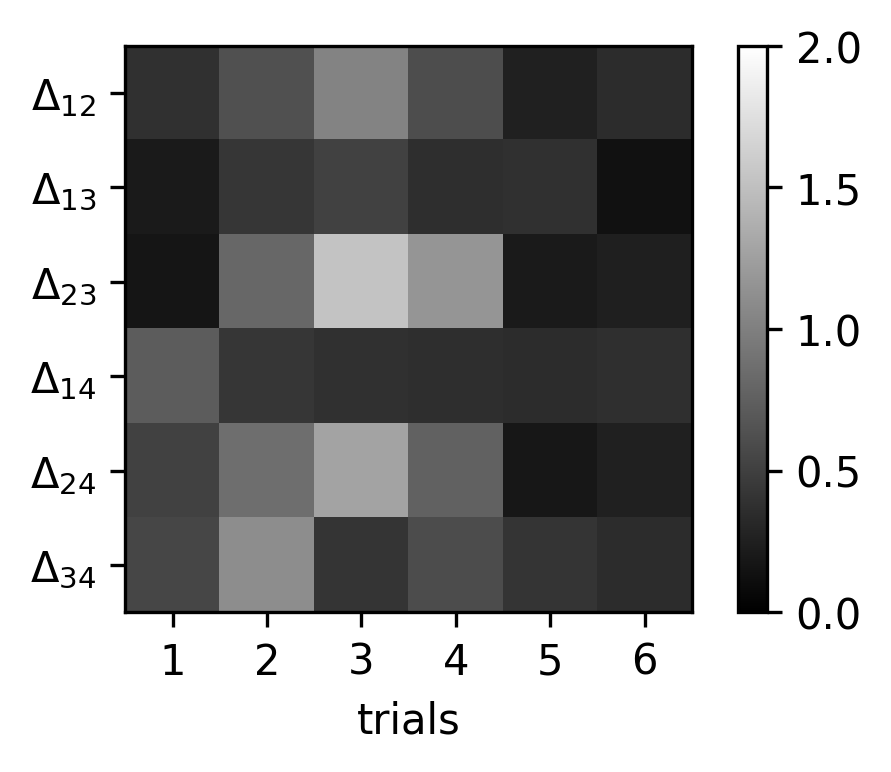}
    \includegraphics[width=0.245\linewidth]{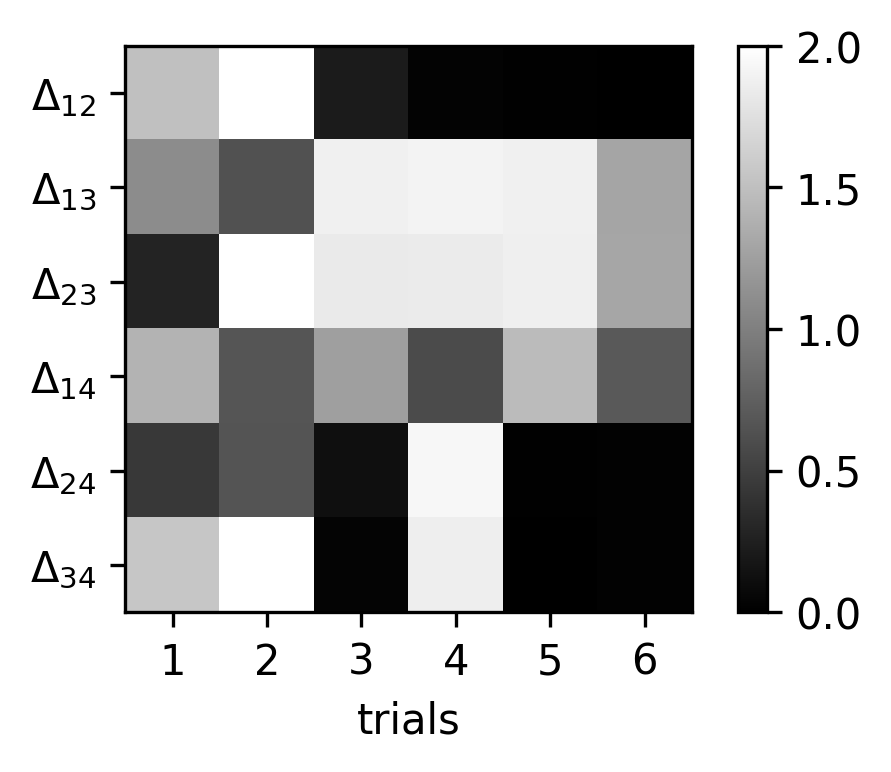}
    \includegraphics[width=0.245\linewidth]{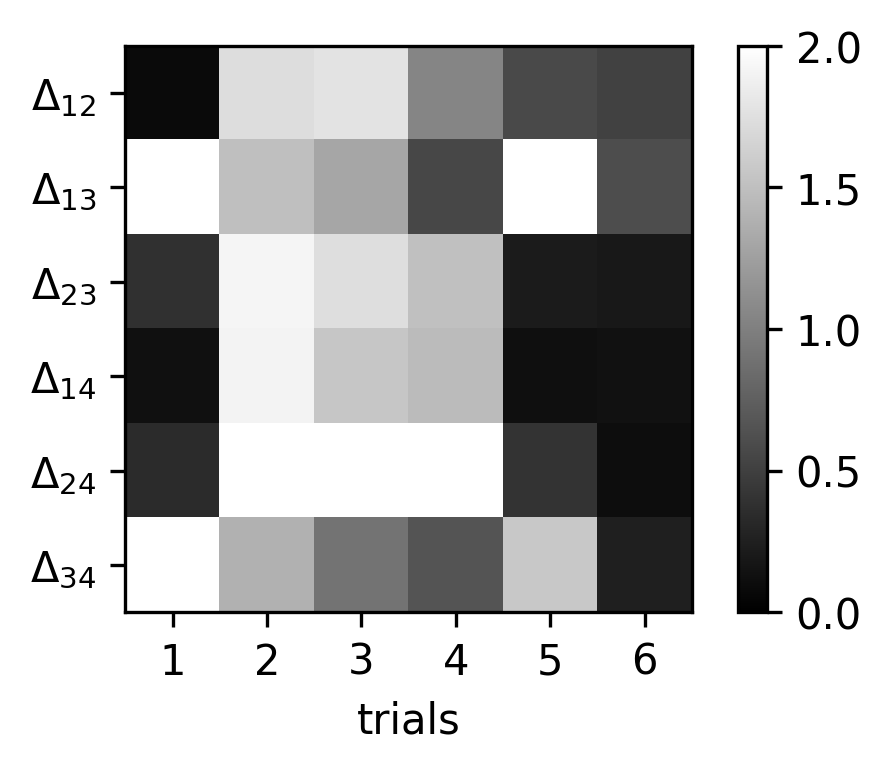}
    \caption{The final distance between the weights of neurons for different teachers and the same initialization. For different tasks, the model converged to a different symmetry class. Left to right: (1) MLP small init., (2) MLP large init., (3) transformer small init., (4) transformer large init.}
    \label{fig:final learning result}
\end{figure}

In Figure \ref{fig:final learning result}, we train two-layer MLPs and simple transformers on synthetic datasets. The tasks match Figure \ref{fig:degree_symmetry}, and we report the pairwise distances among six neurons. Figure \ref{fig:final learning result} suggests that symmetry adapts to the training data in each run.

\clearpage
\section{Theory}
\subsection{Formal Statement of Theorem \ref{theo: complexity}}
We consider the MSE loss for part 2 of the theorem.
\begin{equation}
    \ell(x,\theta)=||y(x)-f(x,\theta)||^2.
\end{equation}
Consider an empirical data distribution $P(x)$, where $x$ contains both the input and the label. The SGD iteration is defined as 
\begin{equation}
    \theta_{t+1}=\theta_t- \eta\nabla_\theta\ell(x,\theta_t),
\end{equation}
where $x \sim P(x)$ and $\eta$ is the learning rate.

The GD iteration is defined as
\begin{equation}
    \theta_{t+1}=\theta_t- \eta\nabla_\theta \E_{x \sim P(x)}[\ell(x,\theta_t)].
\end{equation}

\begin{theorem}
    Let $f$ have the $G$-symmetry for which $P_G^T=P_G$, and $\theta$ be intialized at $\theta_0$ such that $P_G\theta_0 =\theta_0$.
    \begin{enumerate}[noitemsep,topsep=0pt, parsep=0pt,partopsep=0pt, leftmargin=14pt]
    \item For all time steps $t$ under GD or SGD, there exists a model $f'(x,\theta')$ and sequence of parameters $\theta'_t$ such that for all $x$,
    \begin{equation}
        f'(x,\theta_t') = f(x, \theta_t),
    \end{equation}
    where ${\rm dim}(\theta') = {\rm dim}(P_G)$.
        \item The kernalized model, $g(x,\theta) = \lim_{\lambda\to 0} (\lambda^{-1} f(x, \lambda\theta+\theta_0) - f(x, \theta_0))$, converges to
    \begin{equation}
        \theta^* = A^+ \sum_x \nabla_\theta f(x,\theta_0)^T y(x) 
        \label{eq:LeastSquare}
    \end{equation}
    under GD for a sufficiently small learning rate. Here, $A:=P_G\sum_x \nabla_\theta f(x,\theta_0)^T\nabla_\theta f(x,\theta_0)P_G$ and $A^+$ denotes the Moore–Penrose inverse of $A$. 
    \end{enumerate}
\end{theorem}
The second part of the theorem means that in the kernel regime\footnote{Technically, this is the lazy training limit \cite{chizat2018lazy}.}, being at a symmetric solution implies that the feature kernel features are being masked by the projection matrix $
    \nabla_\theta f(x,\theta_0)  \to P_G\nabla_\theta f(x,\theta_0)$,
and learning can only happen given these masks. The proof is a slight generalization of Propostions 2 and 3 in Ref. \cite{ziyin2024remove}.

\begin{proof}

\begin{enumerate}[noitemsep,topsep=0pt, parsep=0pt,partopsep=0pt, leftmargin=14pt]
\item Note that $\ell$ has the $G-$symmetry when $f$ has the $G-$symmetry. For $P_G\theta_0=\theta_0$ and any $\theta$, we have $\ell(x,\theta)=\ell(x,\theta_0+P_G(\theta-\theta_0))$. Taking $\theta\to\theta_0$, we have
\begin{equation}
(I-P_G)\nabla_\theta \ell(x,\theta_0)=0,\label{eq:nable f}
\end{equation}
where we use $P_G^T=P_G$. Therefore, for $\theta_1:=\theta_0+\eta\nabla_\theta \ell(x,\theta_0)$, we still have $P_G\theta_1=\theta_1$.

As $G$ is linear, we can suppose that $\{a_i\}_{i=1}^{\text{dim}(P_G)}$ forms a basis of $\text{im}(P_G)$, and define $f'(x,\theta'):=f(x,\sum_{i=1}^{\text{dim}(V^G)}\theta_i'a_i)$ for ${\rm dim}(\theta') = \text{dim}(P_G)$. By choosing $\theta'_i=\theta^Ta_i$, we have $f'(x,\theta_t') = f(x, \theta_t)$.

\item  By \eqref{eq:nable f}, close to any symmetric point $\theta_0$ (any $\theta_0$ for which $P_G\theta_0 = \theta_0$), for all $x$, we have
\begin{align}
    f(x,\theta) - f(x,\theta_0)
    =&\nabla_\theta f(x,\theta_0)P_G\Delta  + O(\|\Delta\|)^2.
\end{align}
Therefore, $g(x,\theta)$ simplifies to a kernel model \begin{equation}
        g(x,\theta)=\nabla_{\theta_0} f(x,\theta_0) P_G\theta.
    \end{equation}
    Let us consider the squared loss $\ell(\theta)=\sum_x||y(x)-g(x,\theta)||^2$ and denote $A:=\sum_xP_G\nabla_{\theta_0} f(x,\theta_0)^T\nabla_{\theta_0}f(x,\theta_0)P_G$, $b:=P_G\sum_x\nabla_{\theta_0} f(x,\theta_0)^Ty(x)$. The GD iteraiton is 
    \begin{equation}
        \theta^{t+1}=\theta^t-2\eta(A\theta^t-b),
    \end{equation}
    where $\theta^0=0$. If
    \begin{equation}
        \eta<\frac{1}{2\lambda_{max}(A)},
    \end{equation}
    GD converges to
    \begin{equation}
    \begin{aligned}
        \theta^*&=\lim_{t\to\infty}\sum_{k=0}^t(I-2\eta A)^k*2\eta b\\
        &=A^+b,
    \end{aligned}
    \end{equation}
    which is the standard least square solution.
\end{enumerate}
\end{proof}

\subsection{Space Quantization Conjecture}\label{app sec: space quantization conjecture}
\begin{theorem}
Consider the loss $\ell(\theta)=\ell_0(\theta)+\gamma||\theta||^2$  with $\theta:=(\theta_1,\cdots,\theta_k)$ and $\theta_i\in\mathbb{R}^n$ for $1\leq i\leq k$. Assume that $\ell_0(\theta_1,\cdots,\theta_k)$ has the permutation symmetry $\ell_0(\theta_1,\cdots,\theta_i,\cdots,\theta_j,\cdots,\theta_k)=\ell_0(\theta_1,\cdots,\theta_j,\cdots,\theta_i,\cdots,\theta_k)$
for any $1\leq i,j\leq k$ and satisfies 
\begin{equation}
||\nabla_{\theta}\ell_0(\theta_1,\cdots,\theta_i+\theta,\cdots,\theta_j-\theta,\cdots,\theta_k)||\leq K||\theta||^q
\label{eq:q-Lipschitz}
\end{equation}
for $q\geq1$ and any pair of neurons. Moreover, assume that $\inf_\theta\ell_0(\theta)>-\infty$ and that $K$ scales with the number of active neurons $m$ as $K=K_0m^{-\alpha}$. Then, at any global minimum, 
\begin{equation}
    m\leq C\gamma^{-\frac{q+1}{2\alpha}}
\end{equation}
for $n$ large enough, where $C$ is some constant.
\end{theorem}

\begin{proof}
We first consider two vectors $\theta_1\neq\theta_2$. Suppose that the global minimum is at $\ell(\theta_1,\theta_2)$. We would like to compare the loss between $(\theta_1,\theta_2)$ and $(\frac{1}{2}(\theta_1+\theta_2),\frac{1}{2}(\theta_1+\theta_2))$, which gives
\begin{equation}
\begin{aligned}
&\ell(\theta_1,\theta_2)-\ell\left(\frac{1}{2}(\theta_1+\theta_2),\frac{1}{2}(\theta_1+\theta_2)\right)\\
=&\ell_0(\theta_1,\theta_2)-\ell_0\left(\frac{1}{2}(\theta_1+\theta_2),\frac{1}{2}(\theta_1+\theta_2)\right)+\gamma(||\theta_1||^2+\gamma||\theta_2||^2-2\gamma||\frac{1}{2}(\theta_1+\theta_2)||^2)\\
=&\ell_0(\theta_1,\theta_2)-\ell_0\left(\frac{1}{2}(\theta_1+\theta_2),\frac{1}{2}(\theta_1+\theta_2)\right)+\frac{1}{2}\gamma||\theta_1-\theta_2||^2.
\end{aligned}
\end{equation}
For the first term, we have
\begin{equation}
||\ell_0(\theta_1,\theta_2)-\ell_0\left(\frac{1}{2}(\theta_1+\theta_2),\frac{1}{2}(\theta_1+\theta_2)\right)||\leq ||\theta_1-\theta_2||\max_{z\in[0,1]}|f'(z)|||\theta_1-\theta_2||,
\end{equation}
where
\begin{equation}
f(z):=\ell_0\left(\frac{1}{2}(\theta_1+\theta_2)+z(\theta_1-\theta_2),\frac{1}{2}(\theta_1+\theta_2)-z(\theta_1-\theta_2)\right).
\end{equation}
By permutation symmetry, we have
\begin{equation}
f'(0)=\left(\nabla_1\ell_0\left(\frac{1}{2}(\theta_1+\theta_2),\frac{1}{2}(\theta_1+\theta_2)\right)-\nabla_2\ell_0\left(\frac{1}{2}(\theta_1+\theta_2),\frac{1}{2}(\theta_1+\theta_2)\right)\right)^T(\theta_1-\theta_2),
\end{equation}
where $\nabla_1$ and $\nabla_2$ denote the derivative of $\ell_0$ w.r.t. its first and second variable. By the permutation symmetry, we have $\nabla_1\ell_0\left(\frac{1}{2}(\theta_1+\theta_2),\frac{1}{2}(\theta_1+\theta_2)\right)=\nabla_2\ell_0\left(\frac{1}{2}(\theta_1+\theta_2),\frac{1}{2}(\theta_1+\theta_2)\right)$, and thus $f'(0)=0$. 

By \eqref{eq:q-Lipschitz} we have
\begin{equation}
f'(z)\leq Kz^q||\theta_1-\theta_2||^q,
\end{equation}
which gives
\begin{equation}
0\geq\ell(\theta_1,\theta_2)-\ell(\frac{1}{2}(\theta_1+\theta_2),\frac{1}{2}(\theta_1+\theta_2))\geq -K||\theta_1-\theta_2||^{q+1}+\frac{1}{2}\gamma||\theta_1-\theta_2||^2.
\label{eq:ell_theta1_theta_2}
\end{equation}
for the global minimum $\ell(\theta_1,\theta_2)$. Thus we have
\begin{equation}
||\theta_1-\theta_2||\geq\left(\frac{\gamma}{2K}\right)^{1/(q-1)}
\end{equation}
for $q>1$. We conclude that any two vectors should be separated by a distance at least $\left(\frac{\gamma}{2K}\right)^{1/(q-1)}$.

Meanwhile, we have
\begin{equation}
\ell_0(\theta)+\gamma||\theta||^2\leq \ell_0(0),
\end{equation}
which gives
\begin{equation}
||\theta_i||^2\leq\frac{\ell_0(0)-L^*}{\gamma},
\end{equation}
for $1\leq i\leq k$, where $L^*=\inf_\theta\ell_0(\theta)>-\infty$.

Therefore, the problem is to put $m$ points in a $n-$dimensional ball of radius $\sqrt{\frac{\ell_0(0)-L^*}{\gamma}}$,  with pair-wise distance greater than $\left(\frac{\gamma}{2K}\right)^{1/(q-1)}$. Thus we have
\begin{equation}
m\leq2^n\left(\frac{\ell_0(0)-L^*}{\gamma}\right)^{n/2}\left(\frac{\gamma}{2K}\right)^{-n/(q-1)}=C'm^{-\frac{\alpha n}{q-1}}\gamma^{-\frac{q+1}{2(q-1)}n},
\end{equation}
which gives
\begin{equation}
m\leq C\gamma^{-\frac{q+1}{2\alpha}+O(1/n)},
\end{equation}
where $C,C'$ denote some constants.

For $q=1$, \eqref{eq:ell_theta1_theta_2} gives
\begin{equation}
K_0m^{-\alpha}-\frac{1}{2}\gamma\geq0,
\end{equation}
and thus
\begin{equation}
m\leq\left(\frac{\gamma}{2K_0}\right)^{-\frac{1}{\alpha}}.
\end{equation}
\end{proof}

\end{document}